\documentclass{article} % For LaTeX2e
\usepackage{iclr2023_conference,times}

\usepackage{amsmath,amsfonts,bm}

\def\eqref#1{equation~\ref{#1}}
\def\floor#1{\lfloor #1 \rfloor}
\def\1{\bm{1}}

\def\eps{{\epsilon}}

\def\rvx{{\mathbf{x}}}

\def\vs{{\bm{s}}}

\def\vw{{\bm{w}}}
\def\vx{{\bm{x}}}
\def\vy{{\bm{y}}}

\def\mI{{\bm{I}}}

\DeclareMathAlphabet{\mathsfit}{\encodingdefault}{\sfdefault}{m}{sl}
\SetMathAlphabet{\mathsfit}{bold}{\encodingdefault}{\sfdefault}{bx}{n}

\DeclareMathOperator*{\argmax}{arg\,max}
\DeclareMathOperator*{\argmin}{arg\,min}

\usepackage{hyperref}
\usepackage{url}

\usepackage{graphics} % for pdf, bitmapped graphics files
\usepackage{epsfig} % for postscript graphics files
\usepackage{times} % assumes new font selection scheme installed
\usepackage{amsmath} % assumes amsmath package installed
\usepackage{amssymb}  % assumes amsmath package installed
\usepackage{booktabs}
\usepackage{amsthm}
\usepackage{wrapfig}
\usepackage{subfigure}

\usepackage{algorithm}
\usepackage{algorithmicx}
\usepackage{algpseudocode}
\usepackage{enumitem}

\usepackage{amssymb}% http://ctan.org/pkg/amssymb
\usepackage{pifont}% http://ctan.org/pkg/pifont
\newcommand{\cmark}{\ding{51}}%
\newcommand{\xmark}{\ding{55}}%

\usepackage{float}
\usepackage{graphicx}
\usepackage{xspace}
\newcommand{\name}{\textsf{DensePure}~}
\newcommand{\namens}{\textsf{DensePure}}

\newtheorem{theorem}{Theorem}[section]

\newtheorem{definition}[theorem]{Definition}

\newtheorem{remark}{Remark}

\DeclareMathOperator*{\dst}{\displaystyle}

\ifodd 1

\newcommand{\bo}[1]{{\color{orange}Bo: #1}}

\newcommand{\com}[1]{{\color{red}\textbf{Comment}: #1}}%comment of the text
\newcommand{\comtoo}[1]{{\color{purple}\textbf{Comment}: #1}}%comment of the text
\newcommand{\resp}[1]{{\color{cyan}\textbf{Response}: #1}} %responses to comments of the text
\else

\newcommand{\bo}[1]{}

\newcommand{\com}[1]{}
\newcommand{\comtoo}[1]{}
\newcommand{\resp}[1]{}
\fi

\title{\namens: Understanding  Diffusion Models towards Adversarial Robustness }

\author{
Chaowei Xiao\thanks{the first four authors contributed equally} $^{,1,3}$ \quad
Zhongzhu Chen $^{*, 2}$  \quad  Kun Jin $^{*, 2}$   \quad Jiongxiao Wang  $^{*, 1}$ \\
\quad\textbf{Weili Nie} $^{3}$  \quad \textbf{Mingyan Liu}  $^{2}$  \quad \textbf{Anima Anandkumar}$^{3,4}$  \quad \textbf{Bo Li}$^{5}$  \quad \textbf{Dawn Song}$^{6}$ \\
$^{1}$Arizona State University, 
$^{2}$ University of Michigan, Ann Arbor, 
$^{3}$ NVIDIA,
$^{4}$ Caltech,
$^{5}$ UIUC,
$^{6}$ UC Berkeley
}

\iclrfinalcopy % Uncomment for camera-ready version, but NOT for submission.
\begin{document}

\maketitle
\begin{abstract}
    % \input{sections/abs}
      %\textbf{version 2}

Diffusion models have been recently employed to  improve certified robustness through the process of denoising.  However, the theoretical understanding of why diffusion models are able to improve the certified robustness is still lacking, preventing from further improvement.  
% However,  theoretical guarantees for the denoising process are so far  lacking. 
%preventing from further improvement.  
In this study, we close this gap by analyzing the fundamental properties of diffusion models and  establishing the conditions under which they can enhance certified robustness.
This deeper understanding allows us to propose a new method   \namens,  designed to improve the certified robustness of a pretrained model (i.e. classifier).  
Given an (adversarial) input, \name consists of multiple runs of denoising via the reverse process of the diffusion model (with different random seeds) to get multiple reversed samples, which are then passed through the classifier, followed by majority voting of inferred labels to make the final prediction.  
% \name consists of multiple runs of denoising via the reverse process of the diffusion model of perturbed sample (with different random seeds), which are then passed through the classifier of the given model, followed by majority voting of inferred labels to make the final prediction.  
%\name is inspired by our theoretical analysis, 
This design of using multiple runs of denoising is informed by our theoretical analysis of the conditional distribution of the reversed  sample. % In this work, we  delve into diffusion model and propose \name to improve the certified robustness of a standard classifier by properly using the off-the-shelf diffusion models based on our theoretical analysis.
% In this work, we  delve into diffusion model and propose \name to improve the certified robustness of a vanilla model by properly use the off-the-shelf diffusion model based on our theoretical analysis of diffusion models. 
Specifically, when the \textit{data} density of a clean sample is high, its conditional density under the reverse process in a diffusion model is also high;  thus sampling from the latter conditional distribution can purify the adversarial example and return the corresponding clean sample with a high probability. 
% Because high density of clean samples is sufficient to high conditional density of them. 
By using the highest density point in the conditional distribution as the reversed sample, we identify the robust region of a given instance under the diffusion model's reverse process. We show that this robust region is a union of multiple convex sets, and is potentially much larger than the robust regions identified in previous works. 
In practice, \name can approximate the label of the high density region in the conditional distribution so that it can enhance certified robustness. 
% we characterize the conditional distribution of the reversed samples conditioned on the adversarial input. Using the highest density point in the conditional distribution, we characterize the robust region of a given instance under the diffusion model reverse process, which is a union of multiple convex sets and potentially much larger than the robust regions in previous works. 
% Inspired by theoretical analyses, \name is designed to approximate the label of the high density region in the conditional distribution by incorporating two steps. \name (i) performs the reverse process of a diffusion model to obtain a sample of the posterior data distribution conditioned on the  (adversarial) input, and (ii) repeats such reverse process multiple times with majority vote  for final prediction. 
We conduct extensive experiments to demonstrate the effectiveness of \name by evaluating its certified robustness  given a standard model via randomized smoothing. We show that 
% \name achieves  76.6\% certified accuracy on CIFAR-10 under adversarial perturbations within  $L_2$ norm  0.25, an improvement of 3\% over prior methods. 
\name is consistently better than existing methods on ImageNet, with 7\% improvement on average. 
\end{abstract}

\vspace{-1mm}

\section{Introduction}
% \vspace{-1mm}

% 1. diffusion model is good at removing adv noise based on observations
% 2. but it is unclear about the fundamental reason or sufficient conditions under which diffusion models can effectively remove adv noise
% 3. we theoretically analyze, xxxxx (2 conclusions) (intuition: the union of xxxx, better than rs under general cases)
% 4. in practice, we propose to leverage the majority vote to approximate such high density region and therefore improve the certified robustness
% 5. exp results .....

% \weili{Should we emphasize more the certified robustness in the intro?}

Diffusion models have been shown to be a powerful image generation tool \citep{Ho2020DDPM, Song2021ICLR} owing to their iterative diffusion and denoising processes. These models have achieved state-of-the-art performance on sample quality \citep{dhariwal2021diffusion, vahdat2021score} as well as effective mode coverage \citep{song2021maximum}. 
A diffusion model usually consists of two processes: (i) a forward diffusion process that converts data to noise by gradually adding noise to the input, and (ii) a reverse generative process that starts from noise and generates data by denoising one step at a time~\citep{Song2021ICLR}. 

Given the natural denoising property of diffusion models, \emph{empirical} studies have leveraged them to perform adversarial purification~\citep{nie2022diffusion,wu2022guided,carlini2022certified}. For instance, 
\citet{nie2022diffusion} introduce a diffusion model based purification model  \textit{DiffPure}. 
% , which first adds Gaussian noises to the adversarial sample to obtain a diffused sample (via diffusion process), and then solves a reverse stochastic differential equation (SDE) (via reverse process) to remove the adversarial perturbation to recover a ``clean" sample. 
They {\em empirically} show that by carefully choosing the amount of Gaussian noises added during the diffusion process, adversarial perturbations can be removed while preserving the true label semantics. 
Despite the significant empirical results, there is no provable guarantee of the achieved  robustness. 
\citet{carlini2022certified} instantiate the randomized smoothing approach with the diffusion model to offer a {\em provable guarantee}  of model robustness against   $L_2$-norm bounded adversarial example.
% that  model's prediction is robust to  $L_2$-norm bounded adversarial example.
% However, they view the diffusion model as a {\em blackbox} without having a theoretical understanding of why and how the diffusion models contribute to  such nontrivial  certified robustness. 
However, they do not provide a theoretical understanding of why and how the diffusion models contribute to  such nontrivial  certified robustness. 
% so as to provide certified robustness which offers a provable guarantee that a clean model's prediction is robust to  $L_2$-norm bounded adversarial example.
% Despite the empirical success of diffusion models in removing adversarial perturbations, the theoretical understanding for such process is still lacking.
% Despite the success of diffusion models for adversarial robustness, theoretical understanding is still lacking. 
% Thus,  natural questions emerge: \textit{What are the sufficient conditions under which a diffusion model is able to remove adversarial perturbation?} \textit{Is it possible to further improve the purification performance (i.e., certified robustness) based on the fundamental properties of diffusion models?}

% In this work, 
\textbf{Our Approach.}
We theoretically analyze the fundamental properties of diffusion models to understand why and how it enhances certified robustness. This deeper understanding allows us
to propose a new method
% , which then allows us to propose \namens, which improves the certified robustness of a given model by more effectively using the diffusion  model.
% We propose 
\name to improve the certified robustness of any given classifier by more effectively using the diffusion model. 
An illustration of the \name framework is provided in Figure \ref{pipeline}, where it consists of a pretrained diffusion model and a pretrained classifier. \name incorporates two steps: (i) using the reverse process of the diffusion model to obtain a sample of the posterior data distribution conditioned on the adversarial input; and (ii)  
repeating the reverse process multiple times with different random seeds to approximate the label of high density region in the conditional distribution via a majority vote. In particular,  given an adversarial input, we repeatedly feed it into the reverse process of the diffusion model to get multiple reversed examples and feed them into the classifier to get their labels. We then apply the \textit{majority vote} on the set of labels to get the final predicted label.

\name is inspired by our theoretical analysis, where we show that the diffusion model reverse process provides  a conditional distribution of the reversed sample  given an adversarial input, and sampling from this conditional distribution enhances the certified robustness. Specifically, we prove that when the data density of clean samples is high, it is a  sufficient condition for   the conditional density of the reversed samples to be also high. Therefore, in \namens, samples from the conditional distribution can recover the ground-truth labels  with a high probability.

For the convenience of understanding and rigorous analysis, we use the highest density point in the conditional distribution as the deterministic reversed sample for the classifier prediction. 
% The highest density point can be seen as a representative of the high density region in the conditional distribution. 
We show that the robust region for a given sample under the diffusion model's reverse process is the union of multiple convex sets, each surrounding a region around the ground-truth label. 
% we study the robustness radius of the above process.
% We characterize the robustness region of any true prototype such that all adversarial samples will be classified correctly.
% We characterize the robustness region of any data regions with the ground-truth label such that all adversarial samples will be classified correctly.
% we characterize the robustness region of adversarial samples which will be classified correctly.
% We show that such robust region is a union of multiple convex sets, each surrounding a region with the ground-truth label. 
Compared with the robust region of previous work~\citep{Cohen2019ICML}, which only focuses on the neighborhood of {\em one} region with the ground-truth label, such union of multiple convex sets has the potential to provide a much larger robust region. 
Moreover, the characterization implies that the size of robust regions is affected by the relative density and the distance between data regions with the ground-truth label and those with other labels. 
% The influence of these two factors is controlled by the timestamp we choose in the reverse process.  \chaowei{Not sure whether we need the last two sentences. I find it not strictly connected with the previous sentences. It seems that we miss something in the middle. }
% The influence of these two factors is controlled by the timestamp we choose in the reverse process. 

% In practice, inspired by our theoretical analysis, \name is designed to approximate the label of the high density region in the conditional distribution by incorporating two steps: (i) using the reverse process of the diffusion model to obtain a sample of the posterior data distribution conditioned on the adversarial input; and (ii)  
% repeating the reverse process multiple times to approximate the label of high density region in the conditional distribution via majority vote. 
% In particular,  given an adversarial input, we repeatedly feed it into the reverse process of the diffusion model to get multiple reversed examples and feed them into the classifier to get the final prediction. We then apply the \textit{majority vote} on such predictions to get the final predicted label.

We conduct extensive experiments on ImageNet and CIFAR-10 datasets under different settings to evaluate the certifiable robustness of \namens. In particular, we follow the setting from \citet{carlini2022certified} and rely on  randomized smoothing to certify robustness to adversarial perturbations bounded in the $\mathcal{L}_2$-norm.   We show that \name achieves the new state-of-the-art \emph{certified} robustness on the clean model without tuning any model parameters (off-the-shelf). 
% \name obtains significantly better certified robustness results. 
On ImageNet, it achieves a consistently higher certified accuracy than the existing methods among every $\sigma$ at every radius $\epsilon$ , 7\% improvement on average.     
% We noted that although we can characterize the robust region for an input, generally it is non-convex with (possible) infinite constraints, which makes it hard to find a deterministic certification method to provide the robustness radius (lower bound) for each true prototype. Therefore, we resort to the randomized certification method in \cite{Cohen2019ICML}. On the other hand, we also noted that highest density point locating in high-dimensional spaces requires a huge amount of sampling, which is impossible in practice because we have to solve a SDE for every sampling. As an alternative, we proposed a label majority voting step to approximate the highest density point locating step.

\vspace{-5mm}
\begin{figure}[h]
\begin{center}
%\framebox[4.0in]{$\;$}
\includegraphics[width=0.75\linewidth]{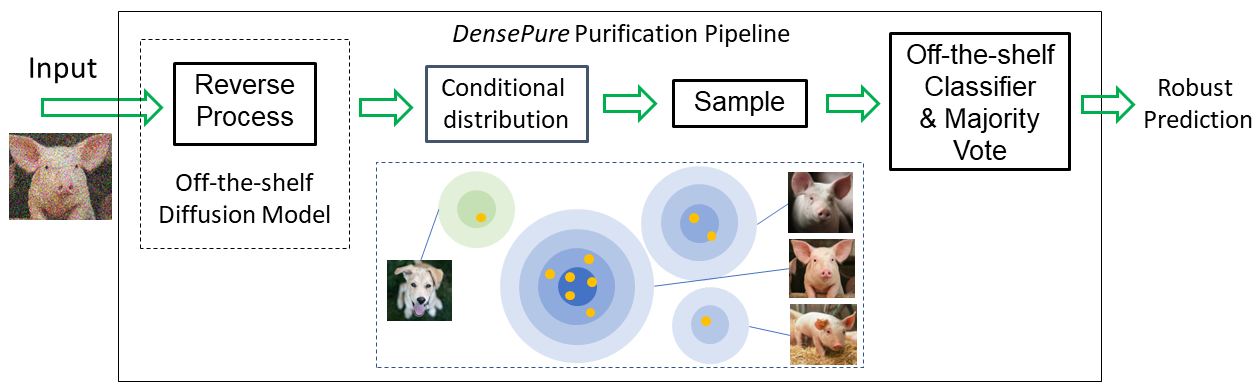}
\end{center}
\vspace{-0.12in}
\caption{Pipeline of \namens.}\label{pipeline}
\vspace{-5mm}
\label{robustfigure}
\end{figure}

% \textbf{\underline{Technical Contributions.}}
% In this paper, we take the first step towards understanding the sufficient conditions of \textit{adversarial purification} with diffusion models.
% We make contributions on both theoretical
% and empirical fronts.
% %%\vspace{-0.5em}
% \begin{itemize}[leftmargin=*]
%     \item We proved that under constrained data density property, an adversarial example can be recovered back to the clean sample with high probability via the reverse process of a diffusion model. 
%     \item By choosing the highest density point in the conditional distribution generated by the reverse process of the diffusion model as the ``clean" sample, we theoretically characterized the robustness region for adversarial samples such that the ``clean" sample has the correct label.
%     \item Based on our analysis and with some compromise due to the computation complexity, we proposed \name in practice, which is a state-of-art adversarial purification pipeline directly leveraging the reverse process of a pre-trained diffusion model and a majority vote step.
%     \item We demonstrated state-of-art performance of \name on xxxx
% \end{itemize}

\textbf{\underline{Technical Contributions.}}
In this paper, we take the first step towards understanding the sufficient conditions of \textit{adversarial purification} with diffusion models.
We make contributions on both theoretical
and empirical fronts: (1) We prove that under constrained data density property, an adversarial example can be recovered back to the original clean sample with high probability via the reverse process of a diffusion model.
    (2) In theory, we characterized the robust region for each point by further taking the highest density point in the conditional distribution generated by the reverse process as the reversed sample. 
    (3) In practice, we proposed \namens, which is a state-of-art adversarial purification pipeline directly leveraging the reverse process of a pre-trained diffusion model and label \textit{majority vote}.
    (4) We demonstrated comparable performance of \name on CIFAR-10 and state-of-the-art performance on ImageNet.

%%\vspace{-0.5em}
% \begin{itemize}[leftmargin=*]
%     \item We prove that under constrained data density property, an adversarial example can be recovered back to the original clean sample with high probability via the reverse process of a diffusion model.
%     \item In theory, we characterized the robust region for each point by further taking the highest density point in the conditional distribution generated by the reverse process as the reversed sample. 
%     \item In practice, we proposed \namens, which is a state-of-art adversarial purification pipeline directly leveraging the reverse process of a pre-trained diffusion model and label \textit{majority vote}.
%     \item We demonstrated state-of-art performance of \name on xxxx \zc{fill in}
% \end{itemize}

% \begin{figure}[h]
% \begin{center}
% %\framebox[4.0in]{$\;$}
% \includegraphics[width=0.85\linewidth]{figures/densepure_robustness_region.png}
% \end{center}
% % \vspace{-0.2in}
% \caption{Base classifier vs \namens.}\label{pipeline}
% \label{robustfigure}
% \end{figure}
% \input{sections/intro_v6}

% \vspace{-1mm}

\section{Preliminaries and Backgrounds}
% \vspace{-1mm}

% \kj{some background here ...

% \begin{enumerate}
%     \item The framework we use for purification, adding the flowchart
%     \item The preliminaries
% \end{enumerate}}

\paragraph{Continuous-Time Diffusion Model.}
% \zc{rewrite this part}
The diffusion model has two components: the \textit{diffusion process} followed by the \textit{reverse process}. Given an input random variable $\rvx_0 \sim p$, the diffusion process adds isotropic Gaussian noises to the data so that the diffused random variable at time $t$ is $\dst \rvx_t = \sqrt{\alpha_t} (\rvx_0 + \boldsymbol{\epsilon}_t)$, s.t., $\dst \boldsymbol{\epsilon}_t \sim \mathcal{N}(\boldsymbol{0}, \sigma_t^2 \mI)$, and $\dst \sigma_t^2 = (1-\alpha_t)/\alpha_t$, and we denote $\dst \rvx_t \sim p_t$. The forward diffusion process can also be defined by the stochastic differential equation
\begin{equation*} \tag{SDE}\label{SDE}
    \dst d \vx = h(\vx, t) dt + g(t) d \vw,
\end{equation*}
where $\dst \vx_0 \sim p$, $\dst h: \mathbb{R}^d \times \mathbb{R} \mapsto \mathbb{R}^d$ is the drift coefficient, $\dst g: \mathbb{R} \mapsto \mathbb{R}$ is the diffusion coefficient, and $\dst \vw(t) \in \mathbb{R}^n$ is the standard Wiener process.
% \zc{cite the notation}
% We define $\dst h(\vx, t) = -\frac{1}{2} \beta(t) \vx, g(t) = \sqrt{\beta(t)}$ and $\alpha_t = e^{-\int_0^t \beta(s) ds}$ the same as \cite{Song2021ICLR}.

Under mild conditions \ref{appendassump}, the reverse process exists and removes the added noise by solving the reverse-time SDE \citep{anderson1982reverse} 
\begin{equation}\tag{reverse-SDE}\label{reverseSDE}
    d \hat{\vx} = [h(\hat{\vx}, t) - g(t)^2 \triangledown_{\hat{\vx}} \log p_t(\hat{\vx})] dt + g(t) d \overline{\vw},
\end{equation}
where $\dst dt$ is an infinitesimal reverse time step, and $\overline{\vw}(t)$ is a reverse-time standard Wiener process.

In our context, we use the conventions of VP-SDE \citep{Song2021ICLR} where $h(\vx; t): =-\frac{1}{2} \gamma(t)x$ and $g(t):= \sqrt{\gamma(t)}$ with $\gamma(t)$ positive and continuous over $[0,1]$, such that $x(t) = \sqrt{{\alpha}_t} x(0)+ \sqrt{1-{\alpha}_t} \boldsymbol{\epsilon}$ where ${\alpha}_t = e^{-\int_0^t \gamma(s)ds}$ and $\boldsymbol{\epsilon}\sim \mathcal{N}(\boldsymbol{0}, \boldsymbol{I})$. We use $\{\rvx_t\}_{t\in [0,1]}$ and $\{\hat \rvx_t\}_{t\in [0,1]}$ to denote the diffusion process and the reverse process generated by \ref{SDE} and \ref{reverseSDE} respectively, which follow the same distribution.

\paragraph{Discrete-Time Diffusion Model (or DDPM \citep{Ho2020DDPM}). } DDPM constructs a discrete Markov chain $\dst \{\rvx_0, \rvx_1, \cdots,  \rvx_i, \cdots, \rvx_N\}$ as the forward process for the training data $\rvx_0 \sim p$, such that $\dst \mathbb{P}(\rvx_i | \rvx_{i-1}) = \mathcal{N}(\rvx_i; \sqrt{1-\beta_i} \rvx_{i-1}, \beta_i I)$, where $\dst 0 < \beta_1 < \beta_2 < \cdots < \beta_N < 1$ are predefined noise scales such that $\rvx_N$ approximates the Gaussian white noise. Denote $\dst \overline{\alpha}_i = \prod_{i=1}^N (1-\beta_i)$, we have $\dst \mathbb{P}(\rvx_i | \rvx_0) = \mathcal{N}(\rvx_i; \sqrt{\overline{\alpha}_i} \rvx_{0}, (1-\overline{\alpha}_i) \mI)$, i.e., $\dst \rvx_t(\rvx_0, \epsilon) = \sqrt{\overline{\alpha}_i} \rvx_{0} +(1-\overline{\alpha}_i) \boldsymbol{\epsilon}, \boldsymbol{\epsilon} \sim \mathcal{N}(\boldsymbol{0},\mI)$.

The reverse process of DDPM learns a reverse direction variational Markov chain $\dst p_{\boldsymbol{\theta}} (\rvx_{i-1} | \rvx_i) = \mathcal{N}(\rvx_{i-1}; \boldsymbol{\mu}_{\boldsymbol{\theta}} (\rvx_i, i), \Sigma_{\boldsymbol{\theta}} (\rvx_i, i))$. \cite{Ho2020DDPM} defines $\dst \boldsymbol{\epsilon}_{\boldsymbol{\theta}}$ as a function approximator to predict $\boldsymbol{\epsilon}$ from $\vx_i$ such that $\dst \boldsymbol{\mu}_{\boldsymbol{\theta}}(\rvx_i, i) = \frac{1}{\sqrt{1-\beta_i}} \left( \rvx_i - \frac{\beta_i}{\sqrt{1-\overline{\alpha}_i}} \boldsymbol{\epsilon}_{\boldsymbol{\theta}}(\rvx_i,i)\right)$. Then the reverse time samples are generated by $\dst \hat{\rvx}_{i-1} = \frac{1}{\sqrt{1-\beta_i}}\left( \hat \rvx_i - \frac{\beta_i}{\sqrt{1-\overline{\alpha}_i}} \boldsymbol{\epsilon}_{\boldsymbol{\theta}^*}(\hat \rvx_i,i) \right) + \sqrt{\beta_i} \boldsymbol{\epsilon}, \boldsymbol{\epsilon} \sim \mathcal{N}(\pmb{0},I)$, and the optimal parameters $\dst \boldsymbol{\theta}^*$ are obtained by solving $\boldsymbol{\theta}^* := \argmin_{\boldsymbol{\theta}} \mathbb{E}_{\rvx_0,\boldsymbol{\epsilon}}\left[ || \boldsymbol{\epsilon} - \boldsymbol{\epsilon}_{\boldsymbol{\theta}}(\sqrt{\overline{\alpha}_i} \rvx_{0} +(1-\overline{\alpha}_i),i) ||_2^2 \right]$.
% \begin{equation*}
%     \boldsymbol{\theta}^* = \argmin_{\boldsymbol{\theta}} \mathbb{E}_{\rvx_0,\boldsymbol{\epsilon}}\left[ || \boldsymbol{\epsilon} - \boldsymbol{\epsilon}_{\boldsymbol{\theta}}(\sqrt{\overline{\alpha}_n} \rvx_{0} +(1-\overline{\alpha}_n),n) ||_2^2 \right].
% \end{equation*}

% \paragraph{Adversarial Samples.}
% % \zc{use boldsymbol for $\boldsymbol{\epsilon}$ and $\vx_0$ for clean sample}
% An adversarial sample adds an adversarial attack $\dst \boldsymbol{\epsilon}_a \in \mathbb{R}^n$ 
% % \footnote{$\boldsymbol(\epsilon)_a$ is used for adversarial attack and the subscript $a$ is never used as a scalar for time in this paper} 
% on a clean sample $\vx_0$, and we denote $\dst \vx_a = (\vx_0 + \boldsymbol{\epsilon}_a)$ as the \emph{adversarial sample} and $\dst \vx_{a,t} = \sqrt{\alpha_t} \vx_a$ the \emph{scaled adversarial sample}, 
% % where we treat $\vx_{a,t}$ as a noisy sample with noise variance $\sigma^2_t$\zc{this is confusing }, 
% and $t \in [0,1]$ is a tunable parameter in our case. In \cite{Cohen2019ICML}, such attacks aim to exploit the vulnerability of the base classifier $f(\cdot) : \mathbb{R}^n \mapsto [1, \dots, N_{class}]$ to achieve $f(\vx_0) \neq f(\vx_a)$ with a bounded budget $||\boldsymbol{\epsilon}_a||_2 \leq B_a$.

\paragraph{Randomized Smoothing.}
Randomized smoothing is used to certify the robustness of a given classifier against $L_2$-norm based perturbation. It transfers the classifier $f$ to a smooth version $ g(\vx) = \argmax_c \mathbb{P}_{\boldsymbol{\epsilon} \sim \mathcal{N}(\boldsymbol{0}, \sigma^2 \boldsymbol{I}) }(f(\vx+\boldsymbol{\epsilon}) = c)$,
% \chaowei{we need to unify the definations of different symbols}:
% \begin{equation*} \label{smooth}
%     g(\vx) = \argmax_c \mathbb{P}_{\boldsymbol{\epsilon} \sim \mathcal{N}(\boldsymbol{0}, \sigma^2 \boldsymbol{I}) }(f(\vx+\boldsymbol{\epsilon}) = c)
% \end{equation*}
where $g$ is the smooth classifier and $\sigma$ is a hyperparameter of the smooth classifier $g$, which  controls the trade-off between robustness and accuracy.   \citet{Cohen2019ICML} shows that $g(x)$ induces the certifiable robustness for $\vx$ under the $L_2$-norm with radius $R$, where $\dst R = \frac{\sigma}{2} \left( \Phi^{-1}(p_A) - \Phi^{-1}(p_B) \right)$; $p_A$ and $p_B$ are probability of the most probable class and ``runner-up''  class respectively; $\Phi$ is the inverse of the standard Gaussian CDF. The $p_A$ and $p_B$ can be estimated with arbitrarily high confidence via Monte Carlo method \citep{Cohen2019ICML}.

% If the attacked sample $\vx_a$ is within the robustness region (or the attack is within the robustness radius as a stronger condition), we are guaranteed to have the same purified classification outcome robust against the attack.

% the conditional distribution of $\dst p(\hat{\rvx}_0 | \hat{\rvx}_T = \vx_a)$ has the highest probability on some $\title{\tilde{\vx}_1} \in \mcS$ when executing the reverse process starting from $\vx_a$ and time $t$ in a backward way until time $0$ such that $f \left(\tilde{\vx}_1\right) = f\left(\vx_0\right)$. 

% \emph{Robustness Radius:}
% If the norm of the difference between attacked sample $\vx_a$ and the true sample is within the robustness radius, the conditional distribution $p(\hat{\rvx}_0 | \hat{\rvx}_t = \vx_a)$ has the highest probability on some $\title{\tilde{\vx}_1} \in \mcS$ when executing the reverse process starting from $\vx_a$ and time $t$ in a backward way until time $0$ such that $f
%  \left(\tilde{\vx}_1\right) = f\left(\vx_0\right)$.

% \vspace{-1mm}
% \section{Robustness for Diffusion Models}
\section{Theoretical Analysis}
% \vspace{-1mm}

In this section, we theoretically analyze why and how the diffusion model can enhance the robustness of a given classifier. 
% In detail, we will show that the diffusion model reverse process will generate a conditional distribution on the adversarial sample. The conditional distribution will have high density on data region with the ground-truth label with high probability, as long as the \textit{data} density of the data region with the ground-truth label is high enough. We further use the highest density point in the conditional distribution as the reversed sample, which is a representative of the high density region. We characterize the robust region such that the reversed sample will have ground-truth label of and only if the adversarial sample locates within the robust region. We show the robust region is the 
We will analyze directly on \ref{SDE} and \ref{reverseSDE} as they generate the same stochastic processes $\{\rvx_t\}_{t\in [0,T]}$ and the literature works establish an approximation on \ref{reverseSDE}~\citep{Song2021ICLR,Ho2020DDPM}.

% we will assume that the reverse process can perfectly recover the diffusion process in the diffusion model, i.e., 
% \begin{align*}
%     \dst \mathbb{P}(\{\rvx_t\}_{t\in I}) =  \dst \mathbb{P}(\{\hat{\rvx}_t\}_{t\in I}), \quad \forall I \subseteq [0,T].
% \end{align*}

We first show that given a diffusion model, solving \ref{reverseSDE} will generate a conditional distribution based on the scaled adversarial sample, which will have high density on data region with high \textit{data} density and near to the adversarial sample in Theorem \ref{distribution:reverse}. See detailed conditions in \ref{appendassump}.

\begin{theorem}\label{distribution:reverse}
Under conditions \ref{appendassump}, solving \eqref{reverseSDE} starting from time $t$ and sample $\dst \vx_{a,t}= \sqrt{\alpha_t} \vx_a$ will generate a reversed random variable $\dst \hat\rvx_0 $ with density $ \dst  \mathbb{P}\left(\hat{\rvx}_0 =\vx| {\hat {\rvx}_t = \vx_{a,t}}\right) \propto p(\vx) \cdot  \frac{1}{\sqrt{\left(2\pi\sigma^2_t\right)^n}} \exp\left({\frac{-|| \vx -\vx_a||^2_2}{2\sigma^2_t}}\right)$,
% \begin{align*}
%     \dst  \mathbb{P}\left(\hat{\rvx}_0 =\vx| {\hat {\rvx}_t = \vx_{a,t}}\right) \propto p(\vx) \cdot  \frac{1}{\sqrt{\left(2\pi\sigma^2_t\right)^n}} e^{\frac{-|| \vx -\vx_a||^2_2}{2\sigma^2_t}}
% \end{align*}
where $p$ is the data distribution, $\dst \sigma_t^2 = \frac{1-\alpha_t}{\alpha_t}$ is the variance of Gaussian noise added at time $\dst t$ in the diffusion process.
\end{theorem}
\vspace{-0.15in}
\begin{proof} (sketch)
Under conditions \ref{appendassump}, we know $\{\rvx_t\}_{t\in [0,1]}$ and $\{\hat \rvx_t\}_{t\in [0,1]}$ follow the same distribution, and then the rest proof follows Bayes' Rule.
% \begin{align*}
%   \dst \mathbb{P}\left(\hat{\rvx}_0 = \vx| {\hat{\rvx}_t = \vx_{a,t}}\right) 
%     \propto&~ \mathbb{P}\left(\rvx_0=\vx\right) \frac{1}{\sqrt{\left(2\pi\sigma^2_t\right)^n}} e^{\frac{-|| \vx -\vx_a||^2_2}{2\sigma^2_t}} = p(\vx) \cdot  \frac{1}{\sqrt{\left(2\pi\sigma^2_t\right)^n}} e^{\frac{-|| \vx -\vx_a||^2_2}{2\sigma^2_t}}.
% \end{align*}
\end{proof}
% \weili{For each theorem, we need to refer to the complete proof in the appendix. }
Please see the full proofs of this and the following theorems in Appendix \ref{app:proofs}.
% (in case we need to save space, I will have one line for each theorem if we have space later)}

\begin{remark}
Note that $\dst \mathbb{P}\left(\hat{\rvx}_0 =\vx| {\hat {\rvx}_t = \vx_{a,t}}\right)>0$ if and only if $p(\vx)>0$, thus the generated reverse sample will be on the data region where we train classifiers.
\end{remark}

%\begin{remark}
In Theorem \ref{distribution:reverse},
the conditional density $\dst \mathbb{P}\left(\hat{\rvx}_0 = \vx| {\hat{\rvx}_t = \vx_{a,t}}\right)$ is high only if both $\dst p(\vx)$ and the Gaussian term have high values, i.e., $\dst \vx$ has high \textit{data} density and is close to the adversarial sample $\vx_a$. The latter condition is reasonable since adversarial perturbations are typically bounded due to budget constraints. Then, the above argument implies that a reversed sample will have the ground-truth label with a high probability if data region with the ground-truth label has high enough \textit{data} density. 
%\end{remark}

% \zc{can we say this? In another angle, in many situations, a small adversarial noise can make a successful attack because the resulted adversarial sample lies outside of data region where the classifier is trained. In this case, the adversarial sample cannot be recognized correctly by the classifier. We note that $\dst \mathbb{P}\left(\hat{\rvx}_0 = \vx| {\hat{\rvx}_t = \vx_{a,t}}\right)$ will be positive only if $\dst p(\vx)$ is positive, which means that the reverse process of diffusion model will bring the adversarial sample back to the data region where the classifier is trained to avoid previous issue.}

For the convenience of theoretical analysis and  understanding, we take the point with highest conditional density $\dst \mathbb{P}\left(\hat{\rvx}_0 = \vx| {\hat{\rvx}_t = \vx_{a,t}}\right)$ as the reversed sample, defined as $\mathcal{P}(\vx_{a};t):=\argmax_{\vx} \mathbb{P}\left(\hat{\rvx}_0 = \vx| {\hat{\rvx}_t = \vx_{a,t}}\right)$. $\mathcal{P}(\vx_{a};t)$ is a representative of the high density data region in the conditional distribution and  $\mathcal{P}(\cdot; t)$ is a deterministic purification model. In the following, we characterize the robust region for data region with ground-truth label under $\dst \mathbb{P}\left(\cdot; t\right)$. The robust region and the robust radius for a general deterministic purification model given a classifier are defined below.
\begin{definition}[Robust Region and Robust Radius]\label{def:robust}
Given a classifier $\dst f$ and a point $\dst \vx_0$, let $\mathcal{G}(\vx_0):=\{\vx: f(\vx)=f(\vx_0)\}$ be the data region where samples have the same label as $\vx_0$. Then given a deterministic purification model $\dst \mathcal{P}(\cdot~; \psi)$ with parameter $\dst \psi$, we define the robust region of $\dst \mathcal{G}(\vx_0)$ under $\dst \mathcal{P}$ and $f$ as $ \dst \mathcal{D}_{\mathcal{P}}^{f}\left(\mathcal{G}(\vx_0); \psi\right):=\left\{\vx : 
    f\left(\mathcal{P}(\vx; \psi)\right)
    % \hat{f}(\vx)
    = f(\vx_0) \right\}$,
% \begin{align*}
%     \dst \mathcal{D}_{\mathcal{P}}^{f}\left(\mathcal{G}(\vx_0); \psi\right):=\left\{\vx : 
%     f\left(\mathcal{P}(\vx; \psi)\right)
%     % \hat{f}(\vx)
%     = f(\vx_0) \right\},
% \end{align*}
i.e., the set of $\vx$ such that purified sample $\dst \mathcal{P}(\vx;\psi)$ has the same label as $\dst \vx_0$ under $f$. Further, we define the robust radius of $\dst \vx_0$ as $ \dst r_{\mathcal{P}}^{f}(\vx_0;\psi):= \max\left\{r: \vx_0+ ru\in \dst \mathcal{D}_{\mathcal{P}}^{f}\left(\vx_0; \psi\right)~, ~ \forall ||u||_2 \le 1 \right\}$,
% \begin{align*}
%     \dst r_{\mathcal{P}}^{f}(\vx_0;\psi):= \max\left\{r: \vx_0+ ru\in \dst \mathcal{D}_{\mathcal{P}}^{f}\left(\vx_0; \psi\right)~, ~ \forall ||u||_2 \le 1 \right\},
% \end{align*}
i.e., the radius of maximum inclined ball of $\dst \mathcal{D}_{\mathcal{P}}^{f}\left(\vx_0; \psi\right)$ centered around $\dst \vx_0$. We will omit $\dst \mathcal{P}$ and $\dst f$ when it is clear from the context and write $\dst \mathcal{D}\left(\mathcal{G}(\vx_0); \psi\right)$ and $\dst r(\vx_0;\psi)$ instead.
\end{definition}
\begin{remark}
In Definition \ref{def:robust}, the robust region (resp. radius) is defined for each class (resp. point). When using the point with highest $\dst \mathbb{P}\left(\hat{\rvx}_0 = \vx| {\hat{\rvx}_t = \vx_{a,t}}\right)$ as the reversed sample, $\psi:=t$.
\end{remark}

Now given a sample $\vx_0$ with ground-truth label, we are ready to characterize the robust region $\dst \mathcal{D}\left(\mathcal{G}(\vx_0); \psi\right)$ under purification model $\mathcal{P}(\cdot ;t)$ and classifier $f$. Intuitively, if the adversarial sample $\dst \vx_a$ is near to $\dst \vx_0$ (in Euclidean distance), $\dst \vx_a$ keeps the same label semantics of $\dst \vx_0$ and so as the purified sample $\dst \mathcal{P}(\vx_a; t)$, which implies that $\dst f\left(\mathcal{P}(\vx_a; \psi)\right) = f(\vx_0)$. However, the condition that $\dst \vx_a$ is near to $\dst \vx_0$ is sufficient but not necessary since we can still achieve $\dst f\left(\mathcal{P}(\vx_a; \psi)\right) = f(\vx_0)$ if $\dst \vx_a$ is near to any sample $\dst \tilde{\vx}_0$ with $\dst f\left(\mathcal{P}(\tilde\vx_a; \psi)\right) = f(\vx_0)$. In the following, we will show that the robust region $\dst \mathcal{D}\left(\mathcal{G}(\vx_0); \psi\right)$ is the union of the convex robust sub-regions surrounding every $\dst \tilde{\vx}_0$ with the same label as $\dst \vx_0$. The following theorem characterizes the convex robust sub-region and robust region respectively.

\begin{theorem}\label{robustregion}
Under conditions \ref{appendassump} and classifier $f$, let $\dst \vx_0$ be the sample with ground-truth label and $\dst \vx_a$ be the adversarial sample, then (i) the purified sample $\dst \mathcal{P}(\vx_a; t)$ will have the ground-truth label if $\dst \vx_a $ falls into the following convex set,
 \begin{align*}
  \dst  \mathcal{D}_{{\tiny\mbox{sub}}}\left(\vx_0;t\right):=\bigcap_{\left\{\vx'_0:f({\vx'_0})\neq f(\vx_0)\right\}} \left\{\vx_a : ({\vx}_a -{\vx_0})^\top ({\vx}'_0-{\vx}_0) < \sigma_t^2 \log\left(\frac{p({\vx}_0)}{p({\vx}'_0)}\right)+\frac{||\vx'_0 -{\vx}_0||^2_2 }{2} \right\},
\end{align*}
and further, (ii) the purified sample $\dst \mathcal{P}(\vx_a; t)$ will have the ground-truth label if and only if $\dst \vx_a $ falls into the following set,
%  \begin{align*}
     $\dst \mathcal{D}\left(\mathcal{G}({\vx}_0);t\right) := \bigcup_{\tilde{{\vx}}_0: f\left(\tilde{{\vx}}_0\right) = f\left({\vx}_0\right)} \mathcal{D}_{{\tiny\mbox{sub}}}\left(\tilde{{\vx}}_0;t\right)$.
%  \end{align*}
 In other words, $\dst \mathcal{D}\left(\mathcal{G}({\vx}_0);t\right)$ is the robust region for data region $\mathcal{G}({\vx}_0)$ under $\dst \mathcal{P}(\cdot ; t)$ and $f$.
\end{theorem}
\begin{proof} (sketch) (i).
% The main idea is to prove that a point $\dst \vx_0'$ such that $\dst f(\vx'_0)\neq f(\vx_0)$ should have lower density than $\dst \vx_0$ in the conditional distribution in Theorem \ref{distribution:reverse} so that $\mathcal{P}(\vx_a; t)$ cannot be $\dst \vx_0'$. In other words, we should have $\mathbb{P}\left(\hat{\rvx}_0 = {\vx}_0| {\hat{\rvx}_t = \vx_{a,t}}\right) >  \mathbb{P}\left(\hat{\rvx}_0 = \vx'_0\mid{\hat{\rvx}_t = \pmb{x}_{a,t}}\right), \forall f(\vx'_0)\neq f(\vx_0)$, which is equivalent to the formula in the lemma.
Each convex half-space defined by the inequality corresponds to a $\vx_0'$ such that  $\dst f(\vx'_0)\neq f(\vx_0)$ where $\vx_a$ within satisfies $\mathbb{P}\left(\hat{\rvx}_0 = {\vx}_0| {\hat{\rvx}_t = \vx_{a,t}}\right) >  \mathbb{P}\left(\hat{\rvx}_0 = \vx'_0\mid{\hat{\rvx}_t = \pmb{x}_{a,t}}\right)$. This implies that $\dst \mathcal{P}(\vx_a; t) \neq \vx_0'$ and $\dst f\left(\mathcal{P}(\vx_a; \psi)\right) = f(\vx_0)$. The convexity is due to that the intersection of convex sets is convex. (ii). The ``if" follows directly from (i). The ``only if" holds because if $\dst \vx_a \notin  \mathcal{D}\left(\mathcal{G}({\vx}_0);t\right)$, then exists $\dst \tilde{{\vx}}_1$ such that $\dst f(\tilde{{\vx}}_1) \neq f({\vx}_0)$ and $\mathbb{P}\left(\hat{\rvx}_0 =\tilde{\vx}_1| {\hat {\rvx}_t = \vx_{a,t}}\right) >  \mathbb{P}\left(\hat{\rvx}_0 =\tilde{\vx}_0| {\hat {\rvx}_t = \vx_{a,t}}\right), \forall \tilde{\vx}_0$ s.t. $\dst f(\tilde{{\vx}}_0) = f({\vx}_0)$, and thus $\dst f\left(\mathcal{P}(\vx_a; \psi)\right) \neq f(\vx_0)$.
\end{proof}

\begin{remark}
Theorem \ref{robustregion} implies that when data region $\mathcal{G}(\vx_0)$ has higher \textit{data} density and larger distances to data regions with other labels, it tends to have larger robust region and points in data region tends to have larger radius.
% Further, the connection $\mathcal{D}_{{\tiny\mbox{sub}}}\left(\tilde{{\vx}}_0;t\right)$ to form $\mathcal{D}\left(\mathcal{G}({\vx}_0);t\right)$ implies we can potentially have larger robust radius compared to the literature work (see Figure \ref{robustfigure}).
\end{remark}

In the literature, people focus more on the robust radius (lower bound) $\dst r\left(\mathcal{G}({\vx}_0);t\right)$ \citep{Cohen2019ICML, carlini2022certified}, which can be obtained by finding the maximum inclined ball inside $\dst \mathcal{D}\left(\mathcal{G}({\vx}_0);t\right)$ centering $\dst {\vx}_0$. Note that although $\dst \mathcal{D}_{{\tiny\mbox{sub}}}\left({\vx}_0;t\right)$ is convex, $\dst \mathcal{D}\left(\mathcal{G}({\vx}_0);t\right)$ is generally not. Therefore, finding $\dst r\left(\mathcal{G}({\vx}_0);t\right)$ is a non-convex optimization problem. In particular, it can be formulated into a disjunctive optimization problem with integer indicator variables, which is typically NP-hard to solve. One alternative could be finding the maximum inclined ball in $\dst \mathcal{D}_{{\tiny\mbox{sub}}}\left({\vx}_0;t\right)$, which can be formulated into a convex optimization problem whose optimal value provides a lower bound for $\dst r\left(\mathcal{G}({\vx}_0);t\right)$. However, $\dst \mathcal{D}\left(\mathcal{G}({\vx}_0);t\right)$ has the potential to provide much larger robustness radius because it might connect different convex robust sub-regions into one, as shown in Figure \ref{robustfigure}.
\vspace{-0.05in}
\begin{figure}[h]
\begin{center}
%\framebox[4.0in]{$\;$}
\includegraphics[width=0.55\linewidth]{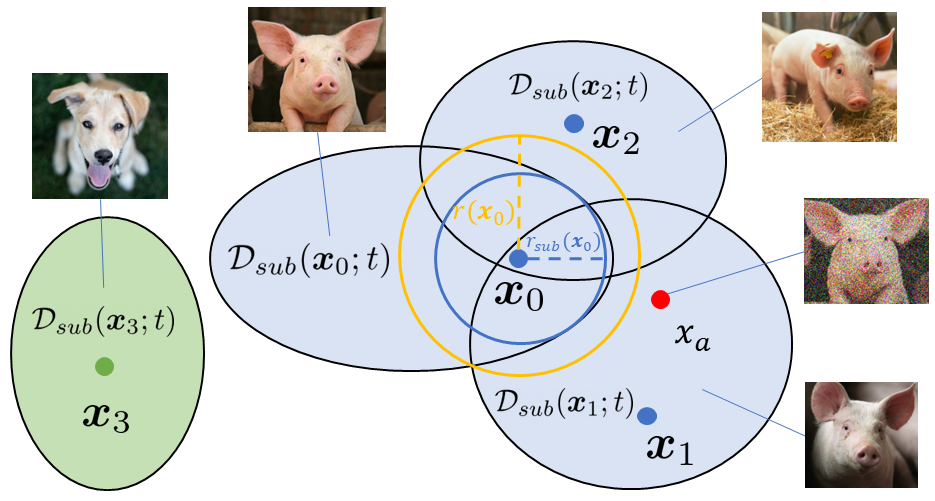}
\end{center}
\vspace{-0.1in}
\caption{An illustration of the robust region $\mathcal{D}(\vx_0; t) = \bigcup_{i=1}^3 \mathcal{D}_{sub}(\vx_i; t)$, where $\vx_0, \vx_1, \vx_2$ are samples with ground-truth label and $\vx_3$ is a sample with another label. $\vx_a = \vx_0+\boldsymbol{\epsilon}_a$ is an adversarial sample such that $\mathcal{P}(\vx_a; t) = \vx_1 \neq \vx_0$ and thus the classification is correct but $\vx_a$ is not reversed back to $\vx_0$. $r_{sub}(\vx_0) < r(\vx_0)$ shows our claim that the union leads to a larger robust radius.}
\vspace{-3mm}

\label{robustfigure}
\end{figure}
\vspace{-0.2in}

In practice, we cannot guarantee to establish an exact reverse process like \ref{reverseSDE} but instead try to establish an approximate reverse process to mimic the exact one. As long as the approximate reverse process is close enough to the exact reverse process, they will generate close enough conditional distributions based on the adversarial sample. Then the density and locations of the data regions in two conditional distributions will not differ much and so is the robust region for each data region. We take the score-based diffusion model in \cite{Song2021ICLR} for an example and demonstrate Theorem \ref{closeconddist} to bound the KL-divergnece between conditional distributions generated by \ref{reverseSDE} and score-based diffusion model. \cite{Ho2020DDPM} showed that using variational inference to fit DDPM is equivalent to optimizing an objective resembling score-based diffusion model with a specific weighting scheme, so the results can be extended to DDPM. 

\begin{theorem}\label{closeconddist}
 Under score-based diffusion model \cite{Song2021ICLR} and conditions \ref{appendassump}, we have $\dst D_{\text{KL}}(\mathbb{P}(\hat \rvx_0 =\vx \mid \hat \rvx_{t} = \vx_{a,t}) \| \mathbb{P}(\rvx^{\theta}_0 =\vx \mid \rvx^{\theta}_{t} = \vx_{a,t})) = \mathcal{J}_{\mathrm{SM}}(\theta, t ; \lambda(\cdot))$,
%  \begin{align*}
%     \dst D_{\text{KL}}(\mathbb{P}(\hat \rvx_0 =\vx \mid \hat \rvx_{t} = \vx_{a,t}) \| \mathbb{P}(\rvx^{\theta}_0 =\vx \mid \rvx^{\theta}_{t} = \vx_{a,t})) = \mathcal{J}_{\mathrm{SM}}(\theta, t ; \lambda(\cdot))
%  \end{align*}
 where $\{\hat \vx_\tau\}_{\tau\in [0,t]}$ and $\{\vx^\theta_\tau\}_{\tau\in [0,t]}$ are stochastic processes generated by \ref{reverseSDE} and score-based diffusion model respectively, $\dst \mathcal{J}_{\mathrm{SM}}(\theta, t ; \lambda(\cdot)):=\frac{1}{2} \int_0^{t} \mathbb{E}_{p_\tau(\mathbf{x})}\left[\lambda(\tau)\left\|\nabla_{\mathbf{x}} \log p_\tau(\mathbf{x})-\boldsymbol{s}_{\theta}(\mathbf{x}, \tau)\right\|_2^2\right] \mathrm{d} \tau,$ $\boldsymbol{s}_{\theta}(\mathbf{x}, \tau)$ is the score function to approximate $\nabla_{\mathbf{x}} \log p_\tau(\mathbf{x})$,  and $\lambda: \mathbb{R}\rightarrow \mathbb{R}$ is any weighting scheme used in the training score-based diffusion models.
\end{theorem}
\vspace{-0.15in}
 \begin{proof}(sketch)
 Let $\dst \boldsymbol{\mu}_{t}$ and $\dst \boldsymbol{\nu}_{t}$ be the path measure for reverse processes $\dst \{\hat \rvx_\tau\}_{\tau\in [0,t]}$ and $\dst \{\rvx^\theta_\tau\}_{\tau\in [0,t]}$ respectively based on the $\vx_{a, t}$.
 Under conditions \ref{appendassump}, $\dst \boldsymbol{\mu}_{t}$ and $\dst \boldsymbol{\nu}_{t}$ are uniquely defined and  the KL-divergence can be computed via the Girsanov theorem \cite{oksendal2013stochastic}.
 \end{proof}
 \begin{remark}
 Theorem \ref{closeconddist} shows that if the training loss is smaller, the conditional distributions generated by \ref{reverseSDE} and score-based diffusion model are closer, and are the same if the training loss is zero.
 \end{remark}

\vspace{-1mm}
\section{\name }
% \vspace{-1mm}

Inspired by the theoretical analysis, we introduce \name  and show how to calculate its  certified robustness radius via the randomized smoothing algorithm. 
% Note that in the experimental section, we will compare our method with \citet{carlini2022certified}, which is one of the most 
% Since it is hard to find the density region, here we empirically leverage the majority vote to estimate it.  

\textbf{Framework.}
Our framework, \namens,  consists of two components: (1)  an off-the-shelf diffusion model with reverse process $ \mathbf{rev}$ and (2) an off-the-shelf base classifier $f$.
% \kj{we called it base classifier $f$ previously, this term is also used with emph later, but emph should be used when it is first introduced} .  
% Given an input, we feed the input  into  the reverse process of the diffusion model to obtain a sample of the posterior data distribution conditioned on the input adversarial example. Then, we perform the reverse process multiple times to approximate the label with high density region via majority vote. 
  
% \name incorporates two steps for adversarial robustness. First, we feed the input  into  the reverse process of the diffusion model to obtain a posterior data distribution conditioned on the input adversarial example. Second, we perform the reverse process multiple times to approximate the label with high density region via majority vote. 

% In particular,  
The pipeline of \name is shown in Figure~\ref{pipeline}. Given an input $\vx$, 
% \kj{Is it a value $\vx$ or a r.v. $\vx$? we need different notation according to ICLR's notation rules, my comments in this section will be mostly notation related.} (either adversarial or benign), 
we feed it into the reverse process $\mathbf{rev}$ of the diffusion model to get the reversed sample $\mathbf{rev}(\vx)$ and then repeat the above process $K$ times 
% \kj{we previously used $r$ for robustness radius, we can think of another notation to use like $K$ or $N_{rev}$} 
% \zc{Please use $\vx$-like style for points and $\vx$-like style for random variables}
to get $K$ reversed samples $\{\mathbf{rev}(\vx)_1, \cdots, \mathbf{rev}(\vx)_{K}\}$. We feed the above $K$ reversed samples into the classifier to get the corresponding prediction $\{f(\mathbf{rev}(\vx)_1), \cdots, f(\mathbf{rev}(\vx)_{K})\}$ and then apply the \textit{majority vote}, termed $\textbf{MV}$, on these predictions to get the final predicted label $\hat{y} = \textbf{MV}(\{f(\mathbf{rev}(\vx)_1), \cdots, f(\mathbf{rev}(\vx)_{K})\}) = \argmax_c  \sum_{i=1}^{K} \pmb{1} \{f(\mathbf{rev}(\vx)_i) = c\}$ .

\textbf{Certified Robustness of \name with Randomized Smoothing.}
% \zc{Please refer to DDPM instead of diffusion model when talking about some speficic features about DDPM}
% First, similar to \citet{carlini2022certified}, our framework first map the randomized smoothing noise $\sigma$ to the noise used for the forward process of the diffusion model. Specifically, given a randomized smoothing parameter $\sigma$ and data input example $x$, randomized smoothing adds a Gaussian noise $\boldsymbol{\epsilon}$, where $\boldsymbol{\epsilon} \sim \mathcal{N}(0, \sigma^2 I)$ to $x$. We could directly map randomized example $x_{\text{rs}} = x+\boldsymbol{\epsilon}$ with the forward processing of the diffusion model $x_\text{rs} = x_t \sim \mathcal{N}(\sqrt{\alpha_t} x_0, (1-\alpha_t) I) $  as :   
% \begin{equation*}
%     \sigma^2 = \frac{1-\alpha_t}{\alpha_t}
% \end{equation*}
% In this way, it equals to  calculate the corresponding timestep $t$ to satisfy the above equation. \chaowei{@jiongxiao, transfer it to SDE. }
% \zc{Please keep consistent when using "$\backslash$citep, $\backslash$citet", perhaps using "$\backslash$cite" for all}

% In practice, we can calculate the certified robustness of \name via Randomized Smoothing (RS) \citep{Cohen2019ICML}, which offers robustness guarantees for a model  under a $L_2$-norm ball. Thus, 
In this paragraph, we will illustrate the algorithm to calculate certified robustness of \name via RS, which offers robustness guarantees for a model  under a $L_2$-norm ball.

% In practice, we can calculate the certified robustness of \name via Randomized Smoothing (RS) \citep{Cohen2019ICML}, which offers robustness guarantees for a model  under a $L_2$-norm ball. Thus, in this paragraph, we will illustrate the algorithm to calculate certified robustness of \name via RS. 

% show how to calculate the certified robustness of \name via  Randomized Smoothing (RS) \citep{Cohen2019ICML}, which offer provable guarantees that a model is robust under a $L_2$-norm ball.
% , for a large fraction of the examples in the test set. . 
In particular, we follow the similar setting of \citet{carlini2022certified} which uses a DDPM-based diffusion model. 
% , which uses a DDPM-based diffusion model in their formulation and experiments~\citep{nichol2021improved,Ho2020DDPM}.
% Thus, for a better understanding of the differences and fair experimental comparison in the later section, we describe our algorithm  based on the  DDPM-base diffusion model as well. 
% Our algorithm can be adapted to score-based diffusion model. 
% Carlini eqation $\vx_0 = \frac{1}{\sqrt{\overline{\alpha}_n}}(\vx_n-\sqrt{1-\overline{\alpha}_t}\boldsymbol{\epsilon}_{\boldsymbol{\theta}}(\vx_n,n))$
The overall algorithm contains three steps:

% \zc{Please rephrase this paragraph, it is not clear what we really want to show}
(1) Our framework  estimates $n$, the number of steps used for the reverse process of DDPM-based diffusion model. 
% To estimate it\zc{the following sentence is not aimed for estimate $n$}, 
Since Randomized Smoothing~\citep{Cohen2019ICML} adds Gaussian noise $\boldsymbol{\epsilon}$, where $\dst \boldsymbol{\epsilon} \sim \mathcal{N}(\boldsymbol{0}, \sigma^2 \mI)$, to data input $\vx$ to get the randomized data input,
% \zc{Please keep consistent with using $\vx_{\text{rs}}$ and $\vx_{\text{rs}}$. 
% It is kind of confused when sometimes using $\vx_{\text{rs}}$ to denote random variable and sometimes using $\vx_{\text{rs}}$ to denote sample, otherwise you say it clearly} 
$\vx_\text{rs} = \vx+\boldsymbol{\epsilon}$, 
% Thus \zc{it is not a ``thus" relationship}, 
we map between the noise required by the randomized example $\vx_\text{rs}$ and the noise required by the diffused data $\vx_n$ (i.e.,  
$\dst \vx_n \sim \mathcal{N}(\vx_n; \sqrt{\overline{\alpha}_n} \vx_0, (1-\overline{\alpha}_n) \mI)$) with $n$ step 
% \kj{we previously used $t \in [0,1]$ for SDE, and $n \in \mathbb{N}[1,N]$ for the time step subscript in DDPM} 
diffusion processing so that $\overline{\alpha}_n = \frac{1}{1+\sigma^2}$.
% $x_t \sim \sqrt{\alpha_t}\mathcal{N}(x_0, \frac{1-\alpha_t}{\alpha_t} \mI)$ 
% \kj{$\dst \vx_t \sim \mathcal{N}(\vx_t; \sqrt{\alpha_t} \vx_0, (1-\alpha_t) \mI)$} 
% so that $\sigma^2 = \frac{1-\overline{\alpha}_n}{\overline{\alpha}_n}$.
In this way, we can  compute the corresponding timestep $n$, where $  n = \argmin_{s} \{ |\overline{\alpha}_s -  \frac{1}{1+\sigma^2} | \ |~ s\in \{1, 2, \cdots, N\} \}$.
% , where 
% In practice, we normally choose the timestep $n$ which makes $\overline{\alpha}_{n}$ from $\left\{\overline{\alpha}_s | s=1...N \right\}$ an approximation to our computed $\frac{1}{1+\sigma^2}$.\chaowei{I do not understand this sentence.}
% In this way, we can directly compute $\overline{\alpha}_n = \frac{1}{1+\sigma^2}$. After that, the corresponding timestep $n$ can be computed through  $\overline{\alpha}_n$. 
% In practice, we normally choose the timestep $n$ which makes $\overline{\alpha}_{n}$ from $\left\{\overline{\alpha}_s | s=1...N \right\}$ an approximation to our computed $\frac{1}{1+\sigma^2}$.\chaowei{I do not understand this sentence.}

% After that, we can get the corresponding $t$ of SDE by solving the following equation based on the selected schedule $\beta(t)$:
% \begin{equation*}
%     \alpha_t = e^{-\int_0^t\gamma(s)ds}
% \end{equation*}
% Note that for discrete $\beta(t)$ schedule, we choose the timestep $t$ which makes the value of  $\frac{1}{1+\sigma^2}$ in the set $\left\{\sqrt{\alpha_s} | s=1...T \right\}$ nearest to $\sqrt{\alpha_t}$ .

(2). Given the above calculated  timestep $n$, we scale $\vx_{rs}$ 
% \zc{just a suggestion: if we can replace $\vx_{rs}$ with $\vx_a$ to keep consistent with theoretical part?} 
with  $\sqrt{\overline{\alpha}_n}$ to obtain the scaled randomized smoothing sample $\sqrt{\overline{\alpha}_n} \vx_{rs}$.
% as our synthetic $\vx_{a,n}$. 
Then we feed $\sqrt{\overline{\alpha}_n} \vx_{rs}$  into the reverse process of the diffusion model by $K$-times to get 
% \kj{resample means something like bootstrapping, let's say to get the reversed input here} 
the reversed sample set $\{ \hat{\vx}_{0}^1 , \hat{\vx}_{0}^2, \cdots, \hat{\vx}_{0}^i,\cdots ,\hat{\vx}_{0}^K\}$.
% where $\hat{\vx}_{0}^i = \textbf{Reverse}( \sqrt{\alpha_t} x_{rs}, t)$,  
% where  $\hat{\vx}_{0}^i = \underbrace{ \textbf{Reverse}( \cdots \textbf{Reverse}( \textbf{Reverse}( \sqrt{\alpha_t} x_{rs}; 1); 1); \cdots 1)}_{t \text{ steps}}$ and $ \textbf{Reverse}( \sqrt{\alpha_t} x_{rs}; 1)$ means $\sqrt{\alpha_t} x_{rs}$ is reversed with {\em single-step} to $\hat{x_{t-1}}^i$. 
% As mentioned before, \name needs to locate the highest density point in the conditional distribution $\dst  p\left(\hat{\vx}_0 =\vx| {\hat {\vx}_t = \vx_{a,t}}\right)$ characterized by Theorem \ref{distribution:reverse}. 
% However, it is hard to do it in reality since we need a huge amount of samplings to makes a confidential estimation of the densities for all points in the sample space. It is extremely computational expensive in the high dimensional space. Thus, here we propose a simple but efficient strategy: \textit{majority voting}. 

(3). We feed the obtained reversed sample set into a standard \emph{off-the-shelf} classifier $f$ to get the corresponding predicted labels $\{ f(\hat{\vx}_{0}^1), f(\hat{\vx}_{0}^2), \dots, f(\hat{\vx}_{0}^i), \dots ,f(\hat{\vx}_{0}^K)\}$, and apply \textit{majority vote}, denoted $\textbf{MV}({\cdots})$, on these predicted labels to get the final label for $\vx_{rs}$.
% , where  $\hat{y} = \textbf{MV}(\{f(\mathbf{rev}(\vx)_1), \cdots, f(\mathbf{rev}(\vx)_{K})\}) = \argmax_c  \sum_{i=1}^{K} \pmb{1} \{f(\mathbf{rev}(\vx))_i = c\}$. 
% The algorithm also shows in Algo~\ref{}.

% \textbf{Reduced  reverse sampling steps.}

\textbf{Fast Sampling.}
% \footnote{Similar to \citep{carlini2022certified}, we use Fast Sampling strategy from \citep{nichol2021improved}. To have a better understanding of the difference between \citep{carlini2022certified} and ours and  for a fair comparison, we use the discrete diffusion model as follows to describe the Fast Sampling strategy we used. }
To calculate the reversed sample, the standard reverse process of DDPM-based models require repeatedly applying a ``single-step'' operation $n$ times to get the reversed sample $\hat{\vx}_{0}$ (i.e., $\hat{\vx}_{0} = 
\underbrace{ 
\textbf{Reverse}( \cdots \textbf{Reverse}(\cdots \textbf{Reverse}( \textbf{Reverse}(\sqrt{\overline{\alpha}_n} \vx_{rs}; n); n-1); \cdots; i);\cdots 1)
}_{n \text{ steps}}
$). 
Here $\hat{\vx}_{i-1} = \textbf{Reverse}(\hat{\vx}_i; i)$ is equivalent to sample $\hat{\vx}_{i-1}$ from $ \mathcal{N}(\hat{\vx}_{i-1}; \boldsymbol{\mu}_{\boldsymbol{\theta}} (\hat{\vx}_i, i), \boldsymbol{\Sigma}_{\boldsymbol{\theta}} (\hat{\vx}_i, i))$, where
$\dst \boldsymbol{\mu}_{\boldsymbol{\theta}}(\hat{\vx}_i, i) = \frac{1}{\sqrt{1-\beta_i}} \left(\hat{ \vx_i} - \frac{\beta_i}{\sqrt{1-\overline{\alpha}_i}} \boldsymbol{\epsilon}_{\boldsymbol{\theta}}(\hat{\vx}_i,i)\right)$ and 
% . For \citep{nichol2021improved}, which used in practice, we set 
$\boldsymbol{\Sigma}_{\boldsymbol{\theta}} := \exp(v\log\beta_i+(1-v) \log\widetilde{\beta}_i)$. Here $v$ is a parameter learned by DDPM and $\widetilde{\beta}_i=\frac{1-\overline{\alpha}_{i-1}}{1-\overline{\alpha}_i}$.
% \kj{could you please let me know where $\Tilde{\beta}$ is first defined and if $\Tilde{\boldsymbol{\mu}}$ is the learned mean vector $\boldsymbol{\mu}^{\boldsymbol{\theta}}$?}
% \end{equation*}

To reduce the time complexity, we  use  the uniform sub-sampling strategy from~\citet{nichol2021improved}. We uniformly sample a subsequence with size $b$ from the original $N$-step the reverse process. 
Note that \citet{carlini2022certified} set $b=1$ for the ``one-shot'' sampling, in this way, $\hat{\vx}_0 = \frac{1}{\sqrt{\overline{\alpha}_n}}(\vx_n-\sqrt{1-\overline{\alpha}_n}\boldsymbol{\epsilon}_{\boldsymbol{\theta}}(\sqrt{\overline{\alpha}_n} \vx_{rs},n))$ is a deterministic value so that the reverse process does not obtain a posterior data distribution conditioned on the input.
Instead, we can tune the number of the sub-sampled DDPM steps to be larger than one ($b>1$) to sample from a posterior data distribution conditioned on the input. 
The details about the fast sampling are shown in appendix~\ref{sec:fast}.  

\section{Experiments}\label{experiments}
% \vspace{-1mm}

In this section, we use \name to evaluate certified robustness on two standard datasets, CIFAR-10~\citep{krizhevsky2009learning} and ImageNet~\citep{deng2009imagenet}. 
% Comparing with other methods, we find \name has the state-of-art certified $L_2$ robust accuracy on both datasets.

\textbf{Experimental settings}
% \textbf{Models.}
We follow the experimental setting from~\citet{carlini2022certified}. Specifically, for CIFAR-10, we use  the 50-M unconditional improved diffusion model from \citet{nichol2021improved} as the diffusion model. We select  ViT-B/16 model \citet{dosovitskiy2020image} pretrained on ImageNet-21k and finetuned on CIFAR-10 as the classifier, which could achieve 97.9\% accuracy on CIFAR-10.  
For ImageNet, we use the unconditional 256$\times$256 guided diffusion model from \citet{dhariwal2021diffusion} as the diffusion model  and pretrained BEiT large model~\citep{bao2021beit} trained on ImageNet-21k as the classifier, which could achieve 88.6\% top-1 accuracy on validation set of ImageNet-1k. 
% \chaowei{Let's add more models. MLP, CNN}
% \cite{} as our diffusion model. 
% And we choose the ViT-B/16 model \cite{} pretarined on ImageNet-21k and finetuned on CIFAR-10 as our classifer. For ImageNet, we use the unconditional 256$\times$256 guided diffusion model \cite{} and BEiT large model \cite().
We select  three different noise levels $\sigma \in \left\{ 0.25, 0.5, 1.0 \right\}$ for certification. For the parameters of \name{}, 
% The sampling numbers when computing the certified radius are $n = 100000$ for CIFAR-10 and $n = 10000$ for ImageNet. 
% We choose the same subset from \citet{Cohen2019ICML} (00 samples from each dataset) for certification . 
we set  $K = 40$ and $b$ = 10 except the results in ablation study.  The details about the baselines are in the appendix.

\begin{table}[t]
\label{4}
\begin{center}
 \resizebox{\linewidth}{!}{%
\begin{tabular}{lrrrrr|rrrrr}
\toprule
\multicolumn{1}{r}{} &  &\multicolumn{9}{c}{Certified Accuracy at $\eps$(\%)} \\ 
\multicolumn{1}{r}{} &  &\multicolumn{4}{c}{CIFAR-10} &\multicolumn{5}{c}{ImageNet}\\ 
Method              & Off-the-shelf & 0.25        & 0.5        & 0.75        & \multicolumn{1}{r|}{1.0}  &0.5 &1.0 &1.5 &2.0 &3.0      \\ \midrule
PixelDP~\citep{lecuyer2019certified}   &   \xmark        &  $^{(71.0)}22.0$           & $^{(44.0)}2.0$            &-             &-    &  $^{(33.0)}16.0$           & -            &-             &-       &-        \\
RS~\citep{Cohen2019ICML}             &   \xmark        & $^{(75.0)}61.0$            &$^{(75.0)}43.0$            & $^{(65.0)}32.0$            & $^{(65.0)}23.0$   & $^{(67.0)}49.0$            &$^{(57.0)}37.0$            & $^{(57.0)}29.0$           &   $^{(44.0)}19.0$   &$^{(44.0)}12.0$    \\
SmoothAdv ~\citep{salman2019provably}     &   \xmark        &  $^{(82.0)}68.0$        &  $^{(76.0)}54.0$      &      $^{(68.0)}41.0$        &$^{(64.0)}32.0$ &  $^{(63.0)}54.0$        &  $^{(56.0)}42.0$      &      $^{(56.0)}34.0$      & $^{(41.0)}26.0$   &$^{(41.0)}18.0$ \\
% SmoothAdv ~\cite{}         & 74.8        &60.8        &    47.0         &     37.8       \\
Consistency ~\citep{jeong2020consistency}  &   \xmark               & $^{(77.8)}68.8$        & $^{(75.8)}58.1$    &     $^{(72.9)}48.5$        &      $^{(52.3)}37.8$   & $^{(55.0)}50.0$        & $^{(55.0)}44.0$    & $^{(55.0)}34.0$        &$^{(41.0)}24.0$  &$^{(41.0)}17.0$     \\
MACER ~\citep{zhai2020macer}      &   \xmark     & $^{(81.0)}71.0$        &$^{(81.0)}59.0$        &    $^{(66.0)}46.0$         &     $^{(66.0)}38.0$   & $^{(68.0)}57.0$        &$^{(64.0)}43.0$        & $^{(64.0)}31.0$         &   $^{(48.0)}25.0$       &$^{(48.0)}14.0$     \\
Boosting ~\citep{horvath2021boosting}   &   \xmark           & $^{(83.4)}70.6$        &    $^{(76.8)}60.4$        &  $^{(71.6)}\textbf{52.4}$           & $^{(73.0)}\textbf{38.8}$             & $^{(65.6)}57.0$        &    $^{(57.0)}44.6$        & $^{(57.0)}38.4$           & $^{(44.6)}28.6$       &  $^{(38.6)}21.2$ \\
SmoothMix ~\citep{jeong2021smoothmix}    &   \cmark         &  $^{(77.1)}67.9$   &        $^{(77.1)}57.9$    &  $^{(74.2)}47.7$           & $^{(61.8)}37.2$       &  $^{(55.0)}50.0$   &        $^{(55.0)}43.0$    &  $^{(55.0)}38.0$           & $^{(40.0)}26.0$     & $^{(40.0)}17.0$       \\ \midrule
Denoised  ~\citep{salman2020denoised}    &   \cmark         &$^{(72.0)}56.0$       &$^{(62.0)}41.0$ &$^{(62.0)}28.0$ &$^{(44.0)}19.0$  & $^{(60.0)}33.0$       &$^{(38.0)}14.0$ &$^{(38.0)}6.0$ &- &- \\
Lee                ~\citep{lee2021provable} &   \cmark   &60.0&    42.0        &28.0&  19.0   &41.0&    24.0        &11.0& - &-     \\ 
Carlini~\citep{carlini2022certified}       &   \cmark       &$^{(88.0)}73.8 $            &$^{(88.0)}56.2$            &$^{(88.0)}41.6$             &$^{(74.2)}31.0$      &$^{(82.0)}74.0 $            &$^{(77.2.0)}59.8$            &$^{(77.2)}47.0$             &$^{(64.6)}31.0$      & $^{(64.6)}19.0$       \\
\textbf{Ours}             &  \cmark  &$^{(87.6)}$\textbf{76.6}             &$^{(87.6)}$\textbf{64.6}            &$^{(87.6)}{50.4}$             &$^{(73.6)}{37.4}$         &$^{(84.0)}$\textbf{77.8}             &$^{(80.2)}$\textbf{67.0}            &$^{(80.2)}$\textbf{54.6}             &$^{(67.8)}$\textbf{42.2}      &   $^{(67.8)}$\textbf{25.8}        \\ \bottomrule
\end{tabular}}
\end{center}
\caption{Certified accuracy compared with existing works. The certified accuracy at $\epsilon=0$ for each model is in the parentheses. The certified accuracy for each cell is from the respective papers except \citet{carlini2022certified}. Our diffusion model and classifier are the same as \citet{carlini2022certified}, where the off-the-shelf classifier uses  ViT-based architectures  trained on a large dataset (ImageNet-22k).}
\vspace{-3mm}

\label{tbl:cifar}
\end{table}

\begin{figure*}[t] 
\small
\centering
    \begin{minipage}{0.45\linewidth}
        \centering
        \includegraphics[width=\textwidth]{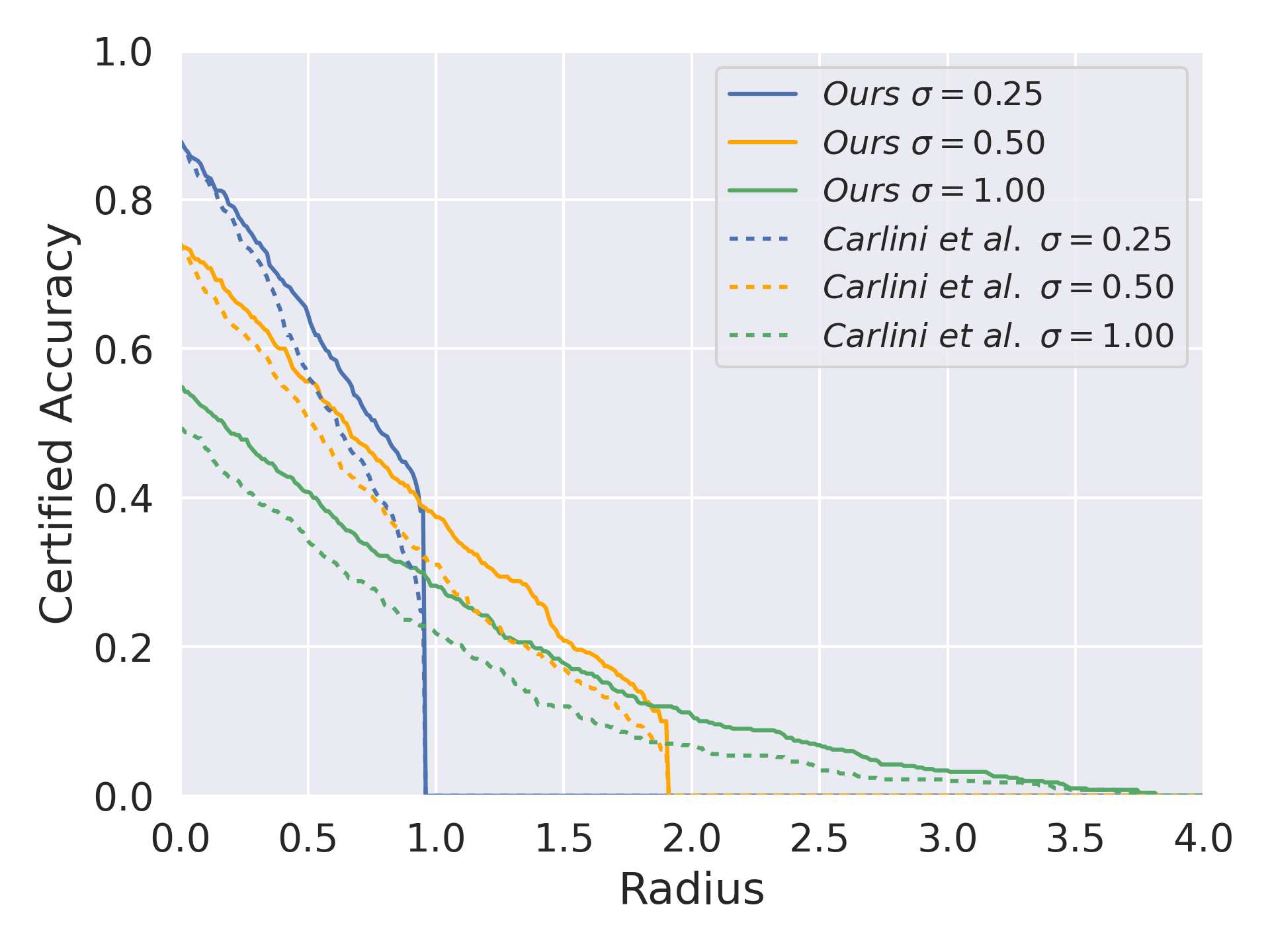}\\
        CIFAR-10
    \end{minipage}\hfill
    \begin{minipage}{0.45\linewidth}
        \centering
        \includegraphics[width=\textwidth]{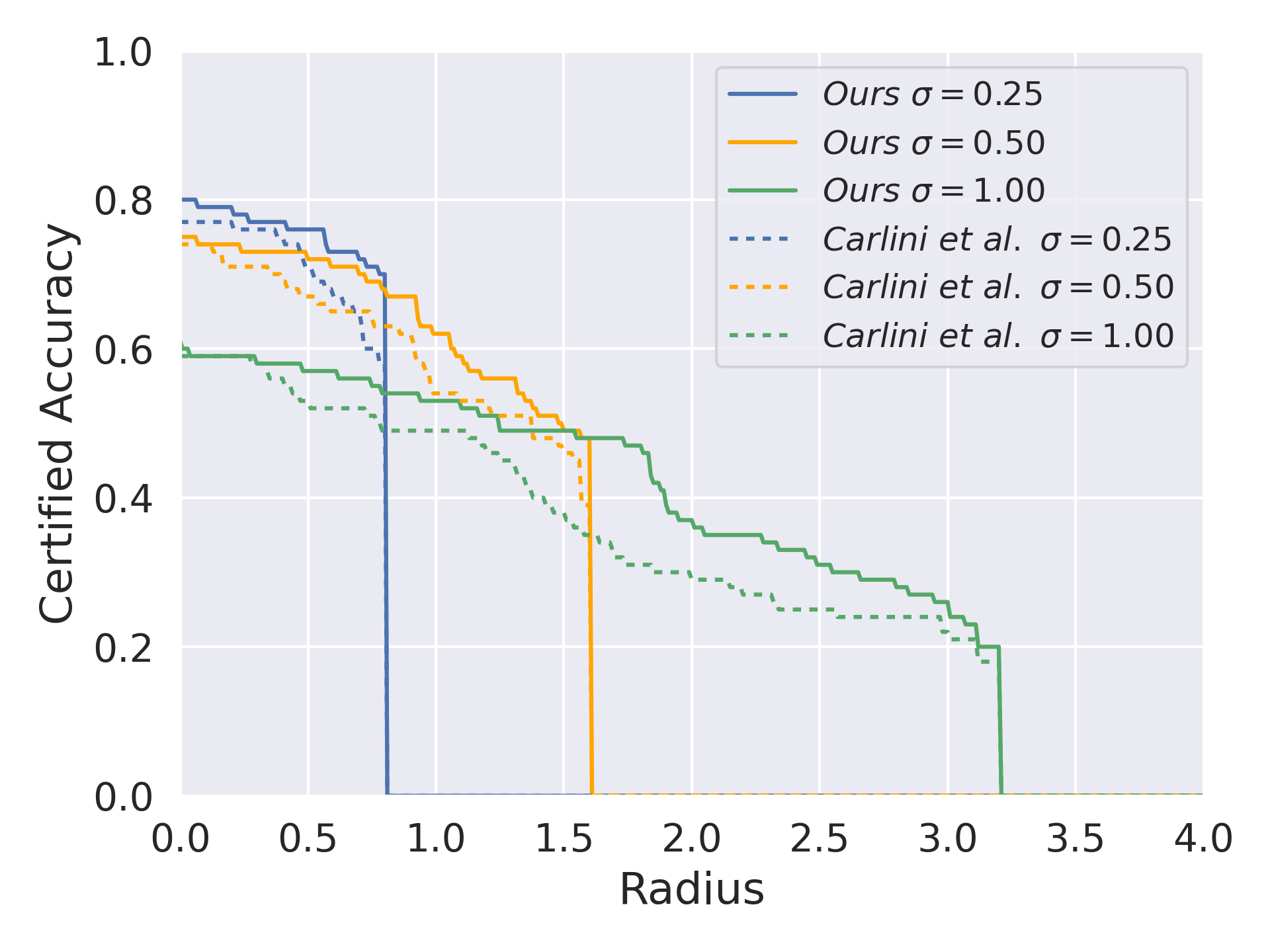}\\
        ImageNet
    \end{minipage}\hfill
    % \begin{minipage}{0.302\linewidth}
    %     \centering
    %     \includegraphics[width=\textwidth]{cifar10_500_comparation.png}\\
    %     (c)
    % \end{minipage}
\vspace{-2mm}
\caption{Comparing our method vs \citet{carlini2022certified} on CIFAR-10 and ImageNet. The lines represent the certified accuracy with different $L_2$ perturbation bound with different Gaussian noise $\sigma \in \{0.25, 0.50, 1.00\}$.}
\vspace{-5mm}
\label{fig:cifarimg}
\end{figure*}
% \vspace{-12mm}
\vspace{-1mm}

\subsection{Main Results}
% \vspace{-1mm}

% We perform \name on the subset of CIFAR-10 or ImageNet. We choose the same subset as in \cite{Cohen2019ICML}, 500 samples for each dataset.
We compare our results with other baselines. The results are shown in Table~\ref{tbl:cifar}.

For CIFAR-10, comparing with the models which are {\em carefully} trained with randomized smoothing techniques in an end-to-end manner (i.e., w/o off-the-shelf classifier), we observe that our method with the standard off-the-shelf classifier outperforms them at smaller $\epsilon= \{0.25, 0.5\}$ on both CIFAR-10 and ImageNet datasets while achieves comparable performance at larger $\epsilon=\{0.75, 1.0\}$. Comparing with the non-diffusion model based methods with off-the-shelf classifier (i.e.,  Denoised~\citep{salman2020denoised} and Lee~\citep{lee2021provable}), both our method and \citet{carlini2022certified} are significantly better than them. 
% Moreover, our method improves the certified robustness over 20\% for CIFAR-10   and \chaowei{xx}\% for ImageNet at every $\epsilon$ radius. 
These results verify the non-trivial adversarial robustness improvements introduced from the diffusion model. 
For ImageNet, our method is consistently better than all priors with a large margin.

Since both \citet{carlini2022certified}  and \name use the diffusion model, 
to better understand the importance of our design, that approximates the label of the high density region in the conditional distribution,  we compare \name with \citet{carlini2022certified}  in a more fine-grained manner. 
% we further compare the performance in a more fine-grained manner to have a better understanding of the importance of our method design from the theoretical analysis.\bo{what does this sentence aim to achieve? maybe make it clear from what aspects we want to show the importance of our method? Or maybe just want to say under diverse smoothing settings?} 

We show detailed certified robustness of the model among different $\sigma$ at different radius for CIFAR-10 in Figure~\ref{fig:cifarimg}-left and for ImageNet in  Figure~\ref{fig:cifarimg}-right. We also present our results of certified accuracy at different $\eps$ in Appendix~\ref{main}.
From these results, we  find that our method is still consistently better  at most $\epsilon$ (except $\epsilon= 0$) among different $\sigma$. 
% Since our method is similar to \citet{carlini2022certified}, we also report the certified accuracy with different $\sigma$ at different $\epsilon$ to better understand the advantage of \name. We could find that our method is better than \citet{carlini2022certified} at $\epsilon>0$ for each $\sigma$.
The performance margin between ours and \citet{carlini2022certified} will become even larger  with a large $\epsilon$. 
% These results show that the diffusion model improving the model robustness but approximating the label with the high density region like \name is the key to exploit the potential of  the diffusion for a better robustness.\chaowei{Zhongzhu.}
These results further indicate that although the diffusion model improves model robustness, leveraging the posterior data distribution conditioned on the input instance (like \name) via reverse process instead of using single sample (\citep{carlini2022certified})  is the key for better robustness. 
Additionally, we use the off-the-shelf classifiers, which are the VIT-based architectures trained a larger dataset. In the later ablation study section, we select the CNN-based architecture wide-ResNet trained on standard dataset from scratch. Our method still achieves non-trivial robustness.

\begin{figure*}[t] 
\small
\centering
    % \begin{minipage}{0.24\linewidth}
    %     \centering
    %     \includegraphics[width=\textwidth]{figures/cifar10_mv.png}\\
    %     CIFAR=10
    % \end{minipage}
    \begin{minipage}{0.43\linewidth}
        \centering
        \includegraphics[width=\textwidth]{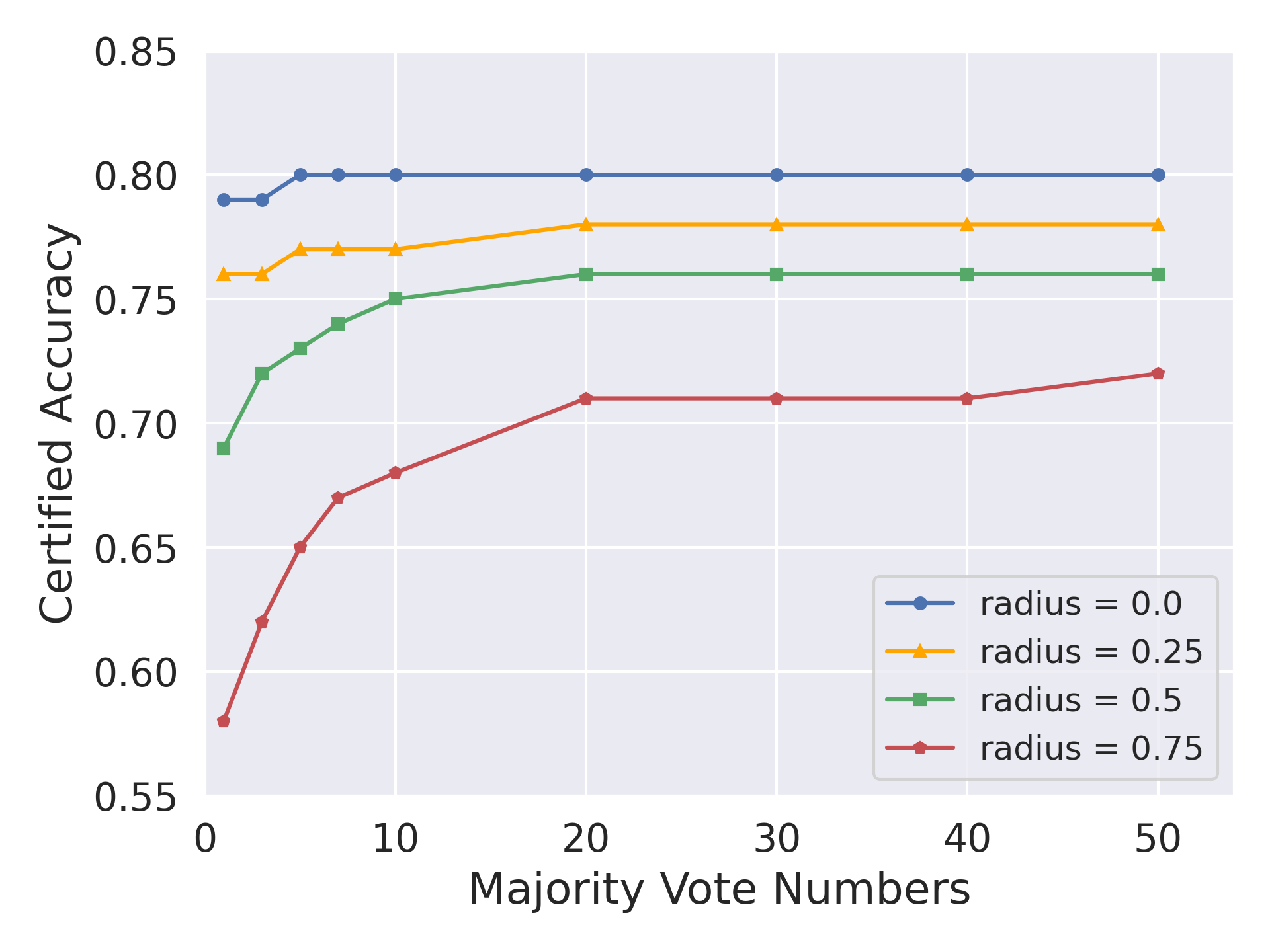}\\
        % ImageNet
    \end{minipage}
    %   \begin{minipage}{0.24\linewidth}
    %     \centering
    %     \includegraphics[width=\textwidth]{figures/cifar10_steps.png}\\
    %     CIFAR=10
    % \end{minipage}
    \begin{minipage}{0.43\linewidth}
        \centering
        \includegraphics[width=\textwidth]{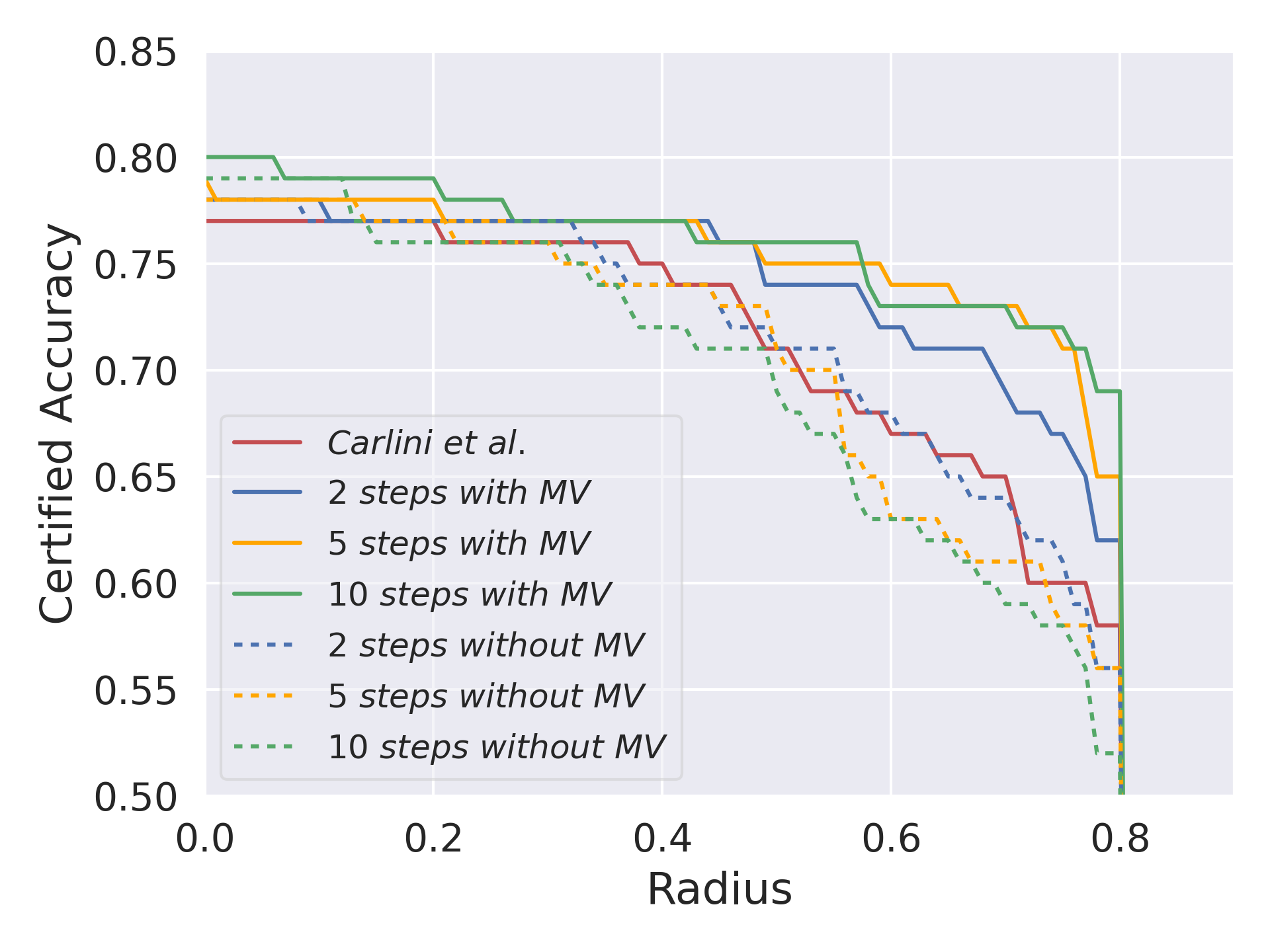}\\
        % ImageNet
    \end{minipage}
\vspace{-2mm}
\caption{Ablation study on ImageNet. The left image shows the certified accuracy among different vote numbers with different radius $\epsilon \in \{0.0, 0.25, 0.5, 0.75\}$. Each line in the figure represents the certified accuracy of our method among different vote numbers $K$ with Gaussian noise $\sigma=0.25$. The right  image shows the certified accuracy with different fast sampling steps $b$. Each line in the figure shows the certified accuracy among different $L_2$ adversarial perturbation bound.}
\vspace{-5mm}

\label{fig:mv-image}
\end{figure*}
% \vspace{-12mm}

% \begin{figure}[h]
% \begin{center}
% %\framebox[4.0in]{$\;$}
% \includegraphics[width=\textwidth]{figures/cifar10_10steps_majority_vote.png}
% \end{center}
% \caption{CIFAR-10 certified accuracy of $\sigma=0.25$ for different voting samples}
% \end{figure}

% \vspace{-1mm}

\subsection{Ablation study}
% \vspace{-1mm}

% There are some unique hyperparameters including number of reversed samples used for \textit{majority vote} ($K$) and  number of reverse steps with Fast Sampling ($b$).  Additionally, one advantage of our method is to use the off-the-shelf classifier, which allows us to analyze the robustness performance with different architectures.
\textbf{Voting samples ($K$)}
We first show how $K$ affects the certified accuracy. For efficiency, we select $b=10$. We conduct experiments for both datasets. We show the certified accuracy among different $r$  at $\sigma=0.25$ in Figure~\ref{fig:mv-image}.  The results for $\sigma=0.5, 1.0$  and CIFAR-10 are shown in the Appendix~\ref{exp:vote}. Comparing with the baseline~\citep{carlini2022certified}, we find that a larger majority vote number leads to a better certified accuracy. 
It verifies that \name indeed benefits the adversarial robustness and  making a good approximation of the label with high density region requires a large number of voting samples.  
We find that our certified accuracy will almost converge at $r=40$. Thus, we set $r=40$ for our experiments. The results with other $\sigma$ show the similar tendency. 
% Thus, we can choose our numbers of majority vote $r$ base on this principle in practical to reach higher certified accuracy with less time comsume.
% To make sure label majority voting can indeed make a contribution to approximate the highest density in the distribution after reverse process, we implement an ablation study on the label majority voting. Figure \ref{fig:mv} shows certified accuracy with different numbers of majority voting and without majority voting for $\sigma=0.25$. These results shows that under some upperbound, majority voting can help us reach a higher certified accuracy.

% \begin{figure}[h]
% \begin{center}
% %\framebox[4.0in]{$\;$}
% \fbox{\rule[-.5cm]{0cm}{4cm} \rule[-.5cm]{4cm}{0cm}}
% \end{center}
% \caption{Detailed CIFAR-10 certified accuracy results comparing with Carlini22 with varying $\sigma$}
% \end{figure}

\textbf{Fast sampling steps ($b$)}
To investigate the role of $b$, we conduct additional experiments with $b\in \{2,5\}$ at $\sigma=0.25$. The results on ImageNet are shown in Figure \ref{fig:mv-image} and results for $\sigma=0.5, 1.0$  and CIFAR-10 are shown in the Appendix~\ref{exp:steps}. By observing results {\em with} majority vote, we find that a larger $b$  can lead to a better certified accuracy since  a larger $b$  generates images with higher quality.  By observing results {\em without} majority vote, the results show opposite conclusions where a larger $b$ leads to a lower certified accuracy, which  contradicts to our intuition.  We guess the potential reason is that though more sampling steps can normally lead to better image recovery quality, it also brings more randomness, increasing the probability that the reversed image locates into a data region with the wrong label.
% ,  especially for large noise adding process in certification. 
% Such randomness can hurt the performance by making the noised data reversed to regions of other classes. 
These results further verify that majority vote is necessary for a better performance.

\textbf{Different architectures}
One advantage of \name is to use the off-the-shelf classifier so that it can plug in any classifier. We choose Convolutional neural network (CNN)-based architectures: Wide-ResNet28-10~\citep{zagoruyko2016wide} for CIFAR-10 with  $95.1\%$ accuracy and Wide-ResNet50-2 for ImageNet with  $81.5\%$ top-1 accuracy,  at $\sigma=0.25$. The results are shown in Table~\ref{tbl:wrn_0.25} and Figure~\ref{fig:modelarch} in Appendix~\ref{exp:models}. Results for more model architectures and $\sigma$ of ImageNet are also shown in Appendix~\ref{exp:models}. We  show that our method can enhance the certified robustness of any given classifier trained on the original data distribution.  Noticeably,  although the performance of  CNN-based classifier is lower than Transformer-based classifier,  \name{} with CNN-based model as the classifier  can outperform  \citet{carlini2022certified} with ViT-based model as the classifier (except $\epsilon=0$ for CIFAR-10).

\begin{table}[t]
\centering
 \resizebox{.9\linewidth}{!}{%
\begin{tabular}{lllrrrrlrrrr}
\toprule
          &     &           \multicolumn{10}{c}{Certified Accuracy at $\boldsymbol{\eps}(\%)$}             \\ 
        %   \cline{3-7} 
Datasets   & Methods & Model  & 0.0 & 0.25 & 0.5  & 0.75 & Model &0.0 &0.25 &0.5 &0.75               \\ \midrule
CIFAR-10 &Carlini~\citep{carlini2022certified} &ViT-B/16 & \textbf{93.0} & 76.0 & 57.0 & \multicolumn{1}{r|}{47.0}  &WRN28-10 & 86.0       & 66.0       & 55.0       & 37.0             \\
 & \textbf{Ours} &ViT-B/16 & 92.0       & \textbf{82.0}       & \textbf{69.0}       & \multicolumn{1}{r|}{\textbf{56.0}}  &WRN28-10 & \textbf{90.0}       & \textbf{77.0}       & \textbf{63.0}       & \textbf{50.0}            \\
\midrule
ImageNet& Carlini~\citep{carlini2022certified} &BEiT & 77.0 & 76.0 & 71.0 & \multicolumn{1}{r|}{60.0} &WRN50-2 & 73.0 & 67.0 & 57.0 & 48.0  \\
 & \textbf{Ours}  &BEiT & \textbf{80.0} & \textbf{78.0} & \textbf{76.0} & \multicolumn{1}{r|}{\textbf{71.0}}  &WRN50-2 & \textbf{81.0} & \textbf{72.0} & \textbf{66.0} & \textbf{61.0}\\
 \bottomrule

\end{tabular}}
\caption{Certified accuracy of our method among different classifier. BeiT and ViT are pre-trained on a larger dataset ImageNet-22k and fine-tuned at ImageNet-1k and CIFAR-10 respectively.  WideResNet is trained on ImageNet-1k for ImageNet and trained on CIFAR-10 from scratch for CIFAR-10. }
\vspace{-5mm}

\label{tbl:wrn_0.25}
\end{table}

% \begin{figure*}[t] 
% \small
% \centering
%     \begin{minipage}{0.45\linewidth}
%         \centering
%         \includegraphics[width=\textwidth]{figures/cifar10_wrn_0.25.png}\\
%         CIFAR=10
%     \end{minipage}\hfill
%     \begin{minipage}{0.45\linewidth}
%         \centering
%         \includegraphics[width=\textwidth]{figures/imagenet_wrn_0.25.png}\\
%         ImageNet
%     \end{minipage}\hfill
%     % \begin{minipage}{0.302\linewidth}
%     %     \centering
%     %     \includegraphics[width=\textwidth]{cifar10_500_comparation.png}\\
%     %     (c)
%     % \end{minipage}
% \vspace{-2mm}
% \caption{certified accuracy of $\sigma=0.25$ for different architectures}
% \label{fig:models}
% \end{figure*}

% \textbf{}
% \begin{figure}[h]
% \begin{center}
% %\framebox[4.0in]{$\;$}
% \fbox{\rule[-.5cm]{0cm}{4cm} \rule[-.5cm]{4cm}{0cm}}
% \end{center}
% \caption{Detailed CIFAR-10 certified accuracy results comparing with Carlini22 with varying $\sigma$}
% \end{figure}

% \begin{figure}[h]
% \begin{center}
% %\framebox[4.0in]{$\;$}
% \fbox{\rule[-.5cm]{0cm}{4cm} \rule[-.5cm]{4cm}{0cm}}
% \end{center}
% \caption{Detailed CIFAR-10 certified accuracy results comparing with Carlini22 with varying $\sigma$}
% \end{figure}

% \vspace{-1mm}

\section{Related Work}
% \vspace{-1mm}

% \paragraph{Adversarial purification}
Using an off-the-shelf generative model to purify adversarial perturbations has become an important direction in adversarial defense. 
Previous works have developed various purification methods based on different generative models, such as GANs~\citep{samangouei2018defense}, autoregressive generative models~\citep{song2018pixeldefend}, and energy-based models~\citep{du2019implicit, grathwohl2020your, hill2021stochastic}. More recently, as diffusion models (or score-based models) achieve better generation quality than other generative models~\citep{Ho2020DDPM,dhariwal2021diffusion}, many works consider using diffusion models for adversarial purification~\citep{nie2022diffusion,wu2022guided,sun2022pointdp} 
% \zc{we are using DDPM in our paper, then why we compare score-based models with DDPM?}. 
Although they have found good empirical results in defending against existing adversarial attacks~\citep{nie2022diffusion}, there is no provable guarantee about the robustness about such methods.
% little understanding of the success of these methods.
% In contrast, we provide a theoretical explanation of their success, from which we develop a more principled defense method using diffusion models.
% \textbf{Certified robustness}
On the other hand, certified defenses provide guarantees of robustness~\citep{mirman2018differentiable,Cohen2019ICML,lecuyer2019certified,salman2020denoised,horvath2021boosting, zhang2018efficient,raghunathan2018certified, raghunathan2018semidefinite,salman2019convex,wang2021beta}. They provide a lower bounder of model accuracy under constrained perturbations. 
% Since it is always challenging to evaluate the model robustness against norm-bounded attacks~\citep{croce2021robustbench}, many works have worked on the certified defense, which provides guarantees of robustness to norm-bounded attacks~\citep{mirman2018differentiable,Cohen2019ICML,lecuyer2019certified,salman2020denoised,horvath2021boosting}.
% PixelDP~\citep{lecuyer2019certified}, RS~\citep{Cohen2019ICML}, SmoothAdv ~\citep{salman2019provably}, Consistency ~\citep{jeong2020consistency}, MACER ~\citep{zhai2020macer}, Boosting ~\citep{horvath2021boosting} , SmoothMix ~\citep{jeong2021smoothmix}, Denoised  ~\citep{salman2020denoised}, Lee~\citep{lee2021provable}, Carlini~\citep{carlini2022certified} 
Among them, approaches~\cite{lecuyer2019certified, Cohen2019ICML,salman2019provably,jeong2020consistency,zhai2020macer,horvath2021boosting,jeong2021smoothmix,salman2020denoised,lee2021provable,carlini2022certified} based on randomized smoothing~\citep{Cohen2019ICML} show the great scalability and achieve promising performance on large network and dataset. 
% More recently, by following denoised smoothing~\citep{salman2020denoised}, a special instantiation of random smoothing that relies on a customized denoiser model, 
The most similar work to us is \citet{carlini2022certified}, which  uses diffusion models combined with standard classifiers  for certified defense. They view diffusion model as blackbox without having a theoretical under-
standing of why and how the diffusion models contribute to such nontrivial certified robustness. 
% In contrast, we theoretically analyze the fundamental properties of diffusion model in why and how it can enhance certified robustness and propose \name to improve the certified robustness of a standard model by properly using the diffusion model based on these properties.
% having a theoretical understanding of the sufficient conditions under which diffusion models can successfully purify noisy images, the certified robustness performance obtained in \cite{carlini2022certified} is worse than our method. 
% \input{sections/conclusion}
% \section{Limitation and Discussion}
% \textbf{Time efficiency }
% \textbf{Time efficiency }

% \vspace{-1mm}

\section{Conclusion}
% \vspace{-1mm}

In this work, we theoretically prove that the diffusion model could purify adversarial examples back to the corresponding clean sample with high probability, as long as the data density of the  corresponding clean samples is high enough. Our theoretical analysis characterizes the conditional distribution of the reversed samples given the adversarial input, generated by the diffusion model reverse process. Using the highest density point in the conditional distribution as the deterministic reversed sample, we identify the robust region of a given instance under the diffusion model reverse process, which is potentially much larger than previous methods. Our analysis inspires us to propose an effective pipeline \namens, for adversarial robustness. We conduct comprehensive experiments to show the effectiveness of \name by evaluating the certified robustness via the randomized smoothing algorithm. Note that \name is an off-the-shelf pipeline that does not require training a smooth classifier. Our results show that  \name achieves the new SOTA certified robustness for perturbation with $\mathcal{L}_2$-norm. We hope that our work sheds light on an in-depth understanding of the diffusion model for adversarial robustness.

\textbf{Limitations.} The time complexity of \name is high since it requires repeating the reverse process multiple times. In this paper, we use fast sampling to reduce the time complexity and show that the setting ($b=2$ and $K=10$) can achieve nontrivial certified accuracy. We leave the more advanced fast sampling strategy as the future direction.

\section*{Ethics Statement}

Our work can positively impact the society by improving the robustness and security of AI systems. We have not involved human subjects or data set releases; instead, we carefully follow the provided licenses of existing data and models for developing and evaluating our method.

\section*{Reproducibility Statement}

For theoretical analysis, all necessary assumptions are listed in \ref{appendassump} and the complete proofs are included in \ref{app:proofs}. The experimental setting and datasets are provided in section \ref{experiments}. The pseudo-code for \name is in \ref{appendpseudocode} and the fast sampling procedures are provided in \ref{sec:fast}.

\bibliographystyle{iclr2023_conference}
\bibliography{iclr2023_conference}

\newpage
\appendix
\appendix

% \section{Appendix}

% \renewcommand{\thetable}{\Alph{table}}
% \renewcommand{\thesection}{\Alph{section}}
% \renewcommand{\theequation}{S\arabic{equation}}
% \renewcommand{\thefigure}{\Alph{figure}}
\section*{\fontsize{16}{16}\selectfont Appendix}
\renewcommand{\thetable}{\Alph{table}}
\renewcommand{\thesection}{\Alph{section}}
\renewcommand{\theequation}{S\arabic{equation}}
\renewcommand{\thefigure}{\Alph{figure}}
\setcounter{figure}{0}
\setcounter{table}{0}
Here is the appendix.

\section{Notations}

\bgroup
\def\arraystretch{1.5}
\begin{tabular}{p{1.5in}p{3.25in}}
$\displaystyle p$  & data distribution\\
$\dst \mathbb{P}(A)$ & probability of event $\dst A$\\
$\displaystyle \mathcal{C}^k$  & set of functions with continuous $k$-th derivatives \\
$\displaystyle\vw(t)$ & standard Wiener Process\\
$\displaystyle\overline{\vw}(t)$ & reverse-time standard Wiener Process\\
$\displaystyle h(\vx,t)$  & drift coefficient in SDE\\
$\displaystyle g(t)$  & diffusion coefficient in SDE\\
$\displaystyle \alpha_t $  & scaling coefficient at time $\dst t$\\
$\displaystyle \sigma_t^2 $  & variance of added Gaussian noise at time $\dst t$\\

$\displaystyle \{\rvx_t\}_{t\in [0,1]}$  & diffusion process generated by SDE\\
$\displaystyle \{\hat \rvx_t\}_{t\in [0,1]}$  & reverse process generated by reverse-SDE\\
$\displaystyle p_t$  & distribution of $\rvx_t$ and $\hat \rvx_t$\\
$\displaystyle \{\rvx_1, \rvx_2,\ldots, \rvx_N\}$  & diffusion process generated by DDPM\\
$\dst \{\beta_i\}_{i=1}^N$ & pre-defined noise scales in DDPM \\
$\displaystyle \boldsymbol{\epsilon}_a$  & adversarial attack\\
$\displaystyle \vx_a$  & adversarial sample\\
$\displaystyle \vx_{a,t}$  & scaled adversarial sample\\
$\displaystyle f(\cdot)$  & classifier\\
$\displaystyle g(\cdot)$  & smoothed classifier\\
$\displaystyle \mathbb{P}\left(\hat{\rvx}_0 =\vx| {\hat {\rvx}_t = \vx_{a,t}}\right)$  & density of conditional distribution generated by reverse-SDE based on $ \vx_{a,t}$\\

$\dst \mathcal{P}(\vx_a; t)$ & purification model with highest density point\\
$\dst \mathcal{G}(\vx_0)$ & data region with the same label as $\vx_0$\\
$\dst \mathcal{D}^f_{\mathcal{P}}(\mathcal{G}(\vx_0);t)$ & robust region for $\dst \mathcal{G}(\vx_0)$ associated with base classifier $f$ and purification model $\dst \mathcal{P}$\\
$\dst r^f_{\mathcal{P}}(\vx_0;t)$ & robust radius for the point associated with base classifier $f$ and purification model $\dst \mathcal{P}$\\
$\dst \mathcal{D}_{sub}(\vx_0;t)$ & convex robust sub-region\\
$\displaystyle \vs_\theta(\vx,t)$  & score function\\
$\displaystyle \{\rvx^{\theta}_t\}_{t\in [0,1]}$  & reverse process generated by score-based diffusion model\\
$\displaystyle \mathbb{P}\left({\rvx}^\theta_0 =\vx| {{\rvx}^\theta_t = \vx_{a,t}}\right)$  & density of conditional distribution generated by score-based diffusion model based on $ \vx_{a,t}$\\
$\lambda(\tau)$ & weighting scheme of training loss for score-based diffusion model\\
$\dst \mathcal{J}_{\mathrm{SM}}(\theta, t ; \lambda(\cdot))$ & truncated training loss for score-based diffusion model\\
$\boldsymbol{\mu}_{t}, \boldsymbol{\nu}_{t}$ & path measure for $\dst \{\hat \rvx_\tau\}_{\tau\in [0,t]}$ and $\dst \{\rvx^\theta_\tau\}_{\tau\in [0,t]}$ respectively\\
\end{tabular}
\egroup
\vspace{0.25cm}
\section{More details about Theoretical analysis }
\subsection{Assumptions}\label{appendassump}
\begin{itemize}
    \item[(i)] The data distribution $\dst p \in \mathcal{C}^2$ and $\mathbb{E}_{\vx\sim p} [||\vx||_2^2]< \infty$.
    \item[(ii)] $\forall t \in[0, T]: h(\cdot, t) \in \mathcal{C}^1, \exists C>0,  \forall \vx \in \mathbb{R}^n, t \in[0, T]:||h(\vx, t)||_2 \leqslant C\left(1+||\vx||_2\right)$.
    \item[(iii)] $\exists C>0, \forall \vx, \vy \in \mathbb{R}^n:||h(\vx, t)-h(\vy, t)||_2 \leqslant C\|\vx-\vy\|_2$.
    \item[(iv)] $g \in \mathcal{C} \text { and } \forall t \in[0, T],|g(t)|>0$.
    \item[(v)] $\forall t \in[0, T]: \vs_\theta(\cdot, t) \in \mathcal{C}^1, \exists C>0,  \forall \vx \in \mathbb{R}^n, t \in[0, T]:||\vs_\theta(\vx, t)||_2 \leqslant C\left(1+||\vx||_2\right)$.
    \item[(vi)] $\exists C>0, \forall \vx, \vy \in \mathbb{R}^n:||\vs_\theta(\vx, t)-\vs_\theta(\vy, t)||_2 \leqslant C\|\vx-\vy\|_2$.
\end{itemize}

\subsection{Theorems and Proofs} \label{app:proofs}

% \begin{theorem*} {
% \subsubsection{Proof of Theorem } \label{app:proof_thm31}
\normalfont 
\textbf{Theorem 3.1.}
% } 
\textit{Under conditions \ref{appendassump}, solving \eqref{reverseSDE} starting from time $t$ and point $\dst \vx_{a,t}= \sqrt{\alpha_t} \vx_a$ will generate a reversed random variable $\dst \hat\rvx_0 $ with conditional distribution
\begin{align*}
    \dst  \mathbb{P}\left(\hat{\rvx}_0 =\vx| {\hat {\rvx}_t = \vx_{a,t}}\right) \propto p(\vx) \cdot  \frac{1}{\sqrt{\left(2\pi\sigma^2_t\right)^n}} e^{\frac{-|| \vx -\vx_a||^2_2}{2\sigma^2_t}}
\end{align*}
where $\dst \sigma_t^2 = \frac{1-\alpha_t}{\alpha_t}$ is the variance of the Gaussian noise added at timestamp $\dst t$ in the diffusion process \ref{SDE}.}
% \end{theorem*}
\begin{proof}
Under the assumption, we know $\{\rvx_t\}_{t\in [0,1]}$ and $\{\hat \rvx_t\}_{t\in [0,1]}$ follow the same distribution, which means
\begin{align*}
   \dst \mathbb{P}\left(\hat{\rvx}_0 = \vx| {\hat{\rvx}_t = \vx_{a,t}}\right) ~=&~ \frac{\mathbb{P}(\hat{\rvx}_0 = \vx,   \hat{\rvx}_t = \vx_{a,t})}{\mathbb{P}(\hat{\rvx}_t = \vx_{a,t})} \\
    = &~ \frac{\mathbb{P}(\rvx_0 = \vx,   \rvx_t = \vx_{a,t})}{\mathbb{P}(\rvx_t = \vx_{a,t})} \\
     = &~ \mathbb{P}\left(\rvx_0=\vx\right)  \frac{ \mathbb{P}(\rvx_t = \vx_{a,t} | \rvx_{0} = \vx)}{\mathbb{P}(\rvx_t = \vx_{a,t})} \\
    \propto&~ \mathbb{P}\left(\rvx_0=\vx\right) \frac{1}{\sqrt{\left(2\pi\sigma^2_t\right)^n}} e^{\frac{-|| \vx -\vx_a||^2_2}{2\sigma^2_t}}\\
    =& ~ p(\vx) \cdot  \frac{1}{\sqrt{\left(2\pi\sigma^2_t\right)^n}} e^{\frac{-|| \vx -\vx_a||^2_2}{2\sigma^2_t}}
\end{align*}
where the third equation is due to the chain rule of probability and the last equation is a result of the diffusion process. 
\end{proof}

% \begin{lemma}
\textbf{Theorem 3.3.}
\textit{
% Under conditions \ref{appendassump}, let $\dst \vx_0$ be the sample with ground-truth label and $\dst \vx_a$ be the adversarial sample, then the purified sample $\dst \mathcal{P}(\vx_a; t)$ will have the ground-truth label if $\dst \vx_a $ falls into the following convex set,
%  \begin{align*}
%   \dst  \mathcal{D}_{{\tiny\mbox{sub}}}\left(\vx_0;t\right):=\bigcap_{f({\vx'_0})\neq f(\vx_0)} \left\{\vx_a : ({\vx}_a -{\vx_0})^\top ({\vx}'_0-{\vx}_0) < \sigma_t^2 \log\left(\frac{p({\vx}_0)}{p({\vx}'_0)}\right)+\frac{1}{2} || \vx'_0 -{\vx}_0||^2_2\right\}.
% \end{align*}
Under conditions \ref{appendassump} and classifier $f$, let $\dst \vx_0$ be the sample with ground-truth label and $\dst \vx_a$ be the adversarial sample, then (i) the purified sample $\dst \mathcal{P}(\vx_a; t)$ will have the ground-truth label if $\dst \vx_a $ falls into the following convex set,
 \begin{align*}
  \dst  \mathcal{D}_{{\tiny\mbox{sub}}}\left(\vx_0;t\right):=\bigcap_{\left\{\vx'_0:f({\vx'_0})\neq f(\vx_0)\right\}} \left\{\vx_a : ({\vx}_a -{\vx_0})^\top ({\vx}'_0-{\vx}_0) < \sigma_t^2 \log\left(\frac{p({\vx}_0)}{p({\vx}'_0)}\right)+\frac{||\vx'_0 -{\vx}_0||^2_2 }{2} \right\},
\end{align*}
and further, (ii) the purified sample $\dst \mathcal{P}(\vx_a; t)$ will have the ground-truth label if and only if $\dst \vx_a $ falls into the following set,
%  \begin{align*}
     $\dst \mathcal{D}\left(\mathcal{G}({\vx}_0);t\right) := \bigcup_{\tilde{{\vx}}_0: f\left(\tilde{{\vx}}_0\right) = f\left({\vx}_0\right)} \mathcal{D}_{{\tiny\mbox{sub}}}\left(\tilde{{\vx}}_0;t\right)$.
%  \end{align*}
 In other words, $\dst \mathcal{D}\left(\mathcal{G}({\vx}_0);t\right)$ is the robust region for data region $\mathcal{G}({\vx}_0)$ under $\dst \mathcal{P}(\cdot ; t)$ and $f$.
}
% \end{lemma}
\begin{proof}
We start with part (i).

The main idea is to prove that a point $\dst \vx_0'$ such that $\dst f(\vx'_0)\neq f(\vx_0)$ should have lower density than $\dst \vx_0$ in the conditional distribution in Theorem \ref{distribution:reverse} so that $\mathcal{P}(\vx_a; t)$ cannot be $\dst \vx_0'$. In other words, we should have
\begin{align*}
     \mathbb{P}\left(\hat{\rvx}_0 = {\vx}_0| {\hat{\rvx}_t = \vx_{a,t}}\right) >  \mathbb{P}\left(\hat{\rvx}_0 = \vx'_0\mid{\hat{\rvx}_t = \pmb{x}_{a,t}}\right).
\end{align*}
By Theorem \ref{distribution:reverse}, this is equivalent to 
\begin{align*}\dst 
     &~p({\vx}_0) \cdot  \frac{1}{\sqrt{\left(2\pi\sigma^2_t\right)^n}} e^{\frac{-|| \vx_0 -\vx_a||^2_2}{2\sigma^2_t}} >  p({\vx}'_0) \cdot  \frac{1}{\sqrt{\left(2\pi\sigma^2_t\right)^n}} e^{\frac{-|| {\vx}'_0 -\vx_a||^2_2}{2\sigma^2_t}}\\
     \Leftrightarrow&~ \log\left(\frac{p(\vx_0)}{p({\vx}'_0)}\right) > \frac{1}{2\sigma^2_t} \left( || \vx_0 -{\vx}_a||^2_2 -|| {\vx}'_0 -{\vx}_a||^2_2\right)\\
     \Leftrightarrow&~ \log\left(\frac{p(\vx_0)}{p({\vx}'_0)}\right) > \frac{1}{2\sigma^2_t} \left( || \vx_0 -{\vx}_a||^2_2 -|| {\vx}'_0 -\vx_0+\vx_0 -{\vx}_a||^2_2\right)\\
     \Leftrightarrow&~  \log\left(\frac{p(\vx_0)}{p({\vx}'_0)}\right) > \frac{1}{2\sigma^2_t} \left( 2({\vx}_a-{\vx}_0)^\top ({\vx}'_0-{\vx}_0)-\|{\vx}'_0-{\vx}_0\|_2^2\right).
\end{align*}
Re-organizing the above inequality, we obtain
\begin{align*}\label{robust:halfspace}
   \dst ({\vx}_a -{\vx}_0)^\top ({\vx}'_0-{\vx}_0) < \sigma_t^2 \log\left(\frac{p({\vx}_0)}{p({\vx}'_0)}\right)+\frac{1}{2} || {\vx}'_0 -{\vx}_0||^2_2.
\end{align*}
Note that the order of $\dst\vx_a$ is at most one in every term of the above inequality, so the inequality actually defines a half-space in $
\dst \mathbb{R}^n$ for every $\dst(\vx_0, \vx'_0)$ pair. Further, we have to satisfy the inequality for every $\dst\vx'_0$ such that $\dst f(\vx'_0)\neq f(\vx_0)$, therefore, by intersecting over all such half-spaces, we obtain a convex $\dst\mathcal{D}_{{\tiny\mbox{sub}}}\left({\vx}_0;t\right)$.
% \end{proof}

% \begin{theorem}
% \textbf{Theorem 3.4.}
% \textit{
%  Under conditions \ref{appendassump} and classifier $f$, let $\dst \vx_0$ be the sample with ground-truth label and $\dst \vx_a$ be the adversarial sample, then the purified sample $\dst \mathcal{P}(\vx_a; t)$ will have the ground-truth label if and only if $\dst \vx_a $ falls into the following set,
%  \begin{align*}
%      \dst \mathcal{D}\left(\mathcal{G}({\vx}_0);t\right) := \bigcup_{\tilde{{\vx}}_0: f\left(\tilde{{\vx}}_0\right) = f\left({\vx}_0\right)} \mathcal{D}_{{\tiny\mbox{sub}}}\left(\tilde{{\vx}}_0;t\right).
%  \end{align*}
%  In other words, $\dst \mathcal{D}\left(\mathcal{G}({\vx}_0);t\right)$ is the robust region for data region $\mathcal{G}({\vx}_0)$ under $\dst \mathcal{P}(\cdot ; t)$ and $f$.}
% \end{theorem}
% \begin{proof}
Then we prove part (ii).

On the one hand, if $\dst \vx_a\in  \mathcal{D}\left(\mathcal{G}({\vx}_0);t\right)$, then there exists one $\dst \tilde{{\vx}}_0$ such that $f(\tilde{\vx}_0) = f(\vx_0)$ and $\dst \vx_a\in \mathcal{D}_{{\tiny\mbox{sub}}}\left(\tilde{{\vx}}_0;t\right)$. By part (i), $\dst \tilde{{\vx}}_0$ has higher probability than all other points with different labels from $\dst {\vx}_0$ in the conditional distribution $\dst  \mathbb{P}\left(\hat{\rvx}_0 =\vx| {\hat {\rvx}_t = \vx_{a,t}}\right)$ characterized by Theorem \ref{distribution:reverse}. Therefore, $\mathcal{P}(\vx_a ; t)$ should have the same label as $\dst {\vx}_0$. On the other hand, if $\dst \vx_a\notin  \mathcal{D}\left(\mathcal{G}({\vx}_0);t\right)$, then there is a point $\dst \tilde{{\vx}}_1$ with different label from $\dst {\vx}_0$ such that for any $\dst \tilde{{\vx}}_0$ with the same label as $\dst {\vx}_0$, $\dst   \mathbb{P}\left(\hat{\rvx}_0 =\tilde{\vx}_1| {\hat {\rvx}_t = \vx_{a,t}}\right) >  \mathbb{P}\left(\hat{\rvx}_0 =\tilde{\vx}_0| {\hat {\rvx}_t = \vx_{a,t}}\right)$. In other words, $\mathcal{P}(\vx_a ; t)$ would have different label from $\dst {\vx}_0$.
\end{proof}

% \begin{theorem}
\textbf{Theorem 3.4.}
\textit{
 Under score-based diffusion model \cite{Song2021ICLR} and conditions \ref{appendassump}, we can bound
 \begin{align*}
    \dst D_{\text{KL}}(\mathbb{P}(\hat \rvx_0 =\vx \mid \hat \rvx_{t} = \vx_{a,t}) \| \mathbb{P}(\rvx^{\theta}_0 =\vx \mid \rvx^{\theta}_{t} = \vx_{a,t})) = \mathcal{J}_{\mathrm{SM}}(\theta, t ; \lambda(\cdot))
 \end{align*}
 where $\{\hat \vx_\tau\}_{\tau\in [0,t]}$ and $\{\vx^\theta_\tau\}_{\tau\in [0,t]}$ are stochastic processes generated by \ref{reverseSDE} and score-based diffusion model respectively, $$\dst \mathcal{J}_{\mathrm{SM}}(\theta, t ; \lambda(\cdot)):=\frac{1}{2} \int_0^{t} \mathbb{E}_{p_\tau(\mathbf{x})}\left[\lambda(\tau)\left\|\nabla_{\mathbf{x}} \log p_\tau(\mathbf{x})-\boldsymbol{s}_{\theta}(\mathbf{x}, \tau)\right\|_2^2\right] \mathrm{d} \tau,$$ $\boldsymbol{s}_{\theta}(\mathbf{x}, \tau)$ is the score function to approximate $\nabla_{\mathbf{x}} \log p_\tau(\mathbf{x})$,  and $\lambda: \mathbb{R}\rightarrow \mathbb{R}$ is any weighting scheme used in the training score-based diffusion models.}
\begin{proof}
Similar to proof of \cite[Theorem 1]{song2021maximum}, let $\dst \boldsymbol{\mu}_{t}$ and $\dst \boldsymbol{\nu}_{t}$ be the path measure for reverse processes $\dst \{\hat \rvx_\tau\}_{\tau\in [0,t]}$ and $\dst \{\rvx^\theta_\tau\}_{\tau\in [0,t]}$ respectively based on the scaled adversarial sample $\vx_{a, t}$. Under conditions \ref{appendassump}, the KL-divergence can be computed via the Girsanov theorem \cite{oksendal2013stochastic}:
\begin{align*}
& ~\dst D_{\text{KL}}\left(\mathbb{P}(\hat \rvx_0 =\vx \mid \hat \rvx_{t} = \vx_{a,t}) \| \mathbb{P}(\rvx^{\theta}_0 =\vx \mid \rvx^{\theta}_{t} = \vx_{a,t})\right)  \\
=&~-\mathbb{E}_{\boldsymbol{\mu}_{t}}\left[\log \frac{d \boldsymbol{\nu}_{t}}{d \boldsymbol{\mu}_{t}}\right] \\
\stackrel{(i)}{=}&~ \mathbb{E}_{\boldsymbol{\mu}_{t}}\left[\int_0^{t} g(\tau)\left(\nabla_{\mathbf{x}} \log p_\tau(\mathbf{x})-\boldsymbol{s}_{\theta}(\mathbf{x}, \tau)\right) \mathrm{d} \overline{\mathbf{w}}_\tau+\frac{1}{2} \int_0^{t} g(\tau)^2\left\|\nabla_{\mathbf{x}} \log p_\tau(\mathbf{x})-\boldsymbol{s}_{\theta}(\mathbf{x}, \tau)\right\|_2^2 \mathrm{~d} \tau\right] \\
=&~ \mathbb{E}_{\boldsymbol{\mu}_{t}}\left[\frac{1}{2} \int_0^{t} g(\tau)^2\left\|\nabla_{\mathbf{x}} \log p_\tau(\mathbf{x})-s_{\theta}(\mathbf{x}, \tau)\right\|_2^2 \mathrm{~d} \tau\right] \\
=&~ \frac{1}{2} \int_0^{\tau} \mathbb{E}_{p_\tau(\mathbf{x})}\left[g(\tau)^2\left\|\nabla_{\mathbf{x}} \log p_\tau(\mathbf{x})-s_{\theta}(\mathbf{x}, \tau)\right\|_2^2\right] \mathrm{d} \tau \\
=& ~\mathcal{J}_{\mathrm{SM}}\left(\theta, t ; g(\cdot)^2\right)
\end{align*}
where (i) is due to Girsanov Theorem  and (ii) is due to the martingale property of Itô integrals.
\end{proof}
\section{More details about \name}
\subsection{Pseudo-Code}\label{appendpseudocode}
We provide the pseudo code of \name in Algo.~\ref{alg1} and Alg.~\ref{alg2}
\begin{algorithm}
\caption{\name pseudo-code with the highest density point}\label{alg1}
\begin{algorithmic}[1]

% \Procedure{Roy}{$a,b$}       \Comment{This is a test}
    \State Initialization: choose off-the-shelf diffusion model and classifier $f$, choose $\dst \psi = t$, 
    \State Input sample $\dst \vx_a = \vx_0 + \boldsymbol{\epsilon}_a$
    \State Compute $\dst \hat{\vx}_0 = \mathcal{P}(\vx_{a}; \psi)$
    \State $\hat{y} = f(\hat{\vx}_0)$
\end{algorithmic}
\end{algorithm}

\begin{algorithm}
\caption{\name pseudo-code with majority vote}\label{alg2}
\begin{algorithmic}[1]

% \Procedure{Roy}{$a,b$}       \Comment{This is a test}
    \State Initialization: choose off-the-shelf diffusion model and classifier $f$, choose $\sigma$  
    \State Compute  $\overline{\alpha}_n = \frac{1}{1+\sigma^2}$,
    $ n = \argmin_{s} \left\{ \left|\overline{\alpha}_s-  \frac{1}{1+\sigma^2} \right| \ |~ s\in \{1, 2, \cdots, N\} \right\}$
    \State Generate input sample $\dst \vx_{\text{rs}} = \vx_0 + \boldsymbol{\epsilon}, \boldsymbol{\epsilon} \sim \mathcal{N}(\boldsymbol{0}, \sigma^2 \mI)$ 
    
    \State Choose schedule $S^b$, get $\dst \hat{\vx}_0^i \leftarrow \mathbf{rev}(\sqrt{\overline{\alpha}_n} \vx_{\text{rs}})_i,  i = 1, 2, \dots, K$ with Fast Sampling
    \State $\hat{y} = \textbf{MV}(\{f(\hat{\vx}_0^1), \dots, f(\hat{\vx}_0^{K})\}) = \argmax_c  \sum_{i=1}^{K} \pmb{1} \{f(\hat{\vx}_0^i) = c\}$
    % \If{$condition = True$}
    %     \State Do this
    %     \If{$Condition \geq 1$}
    %     \State Do that
    %     \ElsIf{$Condition \neq 5$}
    %     \State Do another
    %     \State Do that as well
    %     \Else
    %     \State Do otherwise
    %     \EndIf
    % \EndIf

    % \While{$something \not= 0$}  \Comment{put some comments here}
    %     \State $var1 \leftarrow var2$  \Comment{another comment}
    %     \State $var3 \leftarrow var4$
    % \EndWhile  \label{roy's loop}
% \EndProcedure

\end{algorithmic}
\end{algorithm}

\subsection{Details about Fast Sampling}
\label{sec:fast}

% However, using single step requires running $n$ times reverse process in the diffusion model to get $\hat{\vx}_0$, which is a time-consuming process. 
Applying single-step operation $n$ times is a time-consuming process. In order
to reduce the time complexity, 
we follow the method used in ~\citep{nichol2021improved} and  sample a subsequence $S^b$ with $b$ values (i.e., $S^b= \underbrace{ \{n, \floor{n-\frac{n}{b}}, \cdots, 1 \}}_{b}$ 
% \kj{$\dst |S| = n, |S^b| = b$ if we need to save space later}
, where  $S_j^b$ is the $j$-th element in $S^b$ and $S_j^b=  \floor{n - \frac{jn}{b}}, \forall j < b \text{ and } S_b^b = 1$) from the original schedule $S$ (i.e., $S = \underbrace{ \{n, n-1, \cdots, 1\}}_{n}$, where $S_j= j $ is the $j$-th element in $S$).
% In this way, we reduce the sample schedule $s$  $n$  that the size of $S^b$ is $b$. 
% here we replace the $n$ times ``single-step'' reverse sampling process with %``$\floor{\frac{n}{b}}$-step ($b \ll n$)'' 
% $b$ times sparse reverse sampling strategy\citep{nichol2021improved} by changing the reverse sampling schedule from $S = \underbrace{ \{n, n-1, \cdots, 1\}}_{n}$
% ( where $S_j= j $ is the $j$-th element in $S$)  to 
% $S^b = \underbrace{ \{n, \floor{n-\frac{n}{b}}, \cdots, 1 \}}_{b}$ 
% % \kj{$\dst |S| = n, |S^b| = b$ if we need to save space later} 
% (where  $S_j^b$ is the $j$-th element in $S^b$ and $S_j^b=  \floor{n - \frac{jn}{b}}, \forall j < b \text{ and } S_b^b = 1$).

% $S_j^b = \floor{n - j\frac{n}{b}}$, is the $j$-th element in $S^b$  and the last step is fixed with 1). \kj{$S_j^b =  \floor{n - \frac{jn}{b}}, \forall j < b, S_b^b = 1$}
% (where  $S_j^b = \floor{n - j\frac{n}{b}}$, is the $j$-th element in $S^b$  and the last step is fixed with 1). \kj{$S_j^b =  \floor{n - \frac{jn}{b}}, \forall j < b, S_b^b = 1$}
% $t-b\floor{\frac{t}{b}}$ 
% , so that   $x_{0} = \underbrace{\textbf{Reverse}( \cdots \textbf{Reverse}( \textbf{Reverse}(x_t; \floor{\frac{t}{b}}]); \floor{\frac{t}{b}}]); \cdots ;\floor{\frac{t}{b}})}_{b}$ and $x_$
% \chaowei{The following needs to be changed! }

Within this context, 
% for the new reverse sampling schedule $S^b = \{t, t-[\frac{t}{st}], \cdots, 1 \}$, 
we adapt the original $\overline{\alpha}$ schedule $\overline{\alpha}^S$ = $\{\overline{\alpha}_1, \cdots, \overline{\alpha}_i, \cdots, \overline{\alpha}_n\}$ used for single-step to the new schedule $\overline{\alpha}^{S^b}$ = $\{\overline{\alpha}_{S_1^b}, \cdots, \overline{\alpha}_{S_j^b}, \cdots, \overline{\alpha}_{S_b^b}\}$ (i.e.,  $\overline{\alpha}^{S^{b}}_i = \overline{\alpha}_{S_i^b} = \overline{\alpha}_{S_{\floor{n - \frac{i n}{b}} }}$
is the $i$-th element in $\overline{\alpha}^{S^b}$).  
We calculate the corresponding $\beta^{S^b} = \{\beta^{S^b}_1, \beta^{S^b}_2, \cdots, \beta^{S^b}_i, \cdots,\beta^{S^b}_b  \}$ and $\widetilde{\beta}^{S^b} = \{ \widetilde{\beta}^{S^b}_1, \widetilde{\beta}^{S^b}_2, \cdots, \widetilde{\beta}^{S^b}_i, \cdots, \widetilde{\beta}^{S^b}_b \}$ schedules, where
% \begin{equation*}
 $ \beta_{S^b_i}=\beta^{S^b}_i = 1 - \frac{\overline{\alpha}^{S^{b}}_i }{\overline{\alpha}^{S^{b}}_{i-1}},   \quad \widetilde{\beta}_{S^b_i}=\widetilde{\beta}^{S^{b}}_i = \frac{1-\overline{\alpha}^{S^{b}}_{i-1}}{1-\overline{\alpha}^{S^b}_{i}}\beta_{S^b_i}$.
% \end{equation*}
With these new schedules,  we can use $b$ times reverse steps to  calculate   
$\hat{\vx}_{0} = \underbrace{\textbf{Reverse}( \cdots \textbf{Reverse}( \textbf{Reverse}(\vx_n; S^b_b); S^b_{b-1}); \cdots ; 1)}_{b}$.
% with only $b$ times.
Since $\boldsymbol{\Sigma}_{\boldsymbol{\theta}} (\vx_{S^b_{i}}, S^b_{i})$ is parameterized as a range between $\beta^{S^b}$ and $\widetilde{\beta}^{S^b}$, it will automatically be rescaled. Thus, $\hat{\vx}_{S^b_{i-1}} =  \textbf{Reverse}(\hat{\vx}_{S^b_i}; S^b_i) $  is  equivalent to sample
$\vx_{S^b_{i-1}}$ from $\mathcal{N}(\vx_{S^b_{i-1}}; \boldsymbol{\mu}_{\boldsymbol{\theta}} (\vx_{S^b_{i}}, S^b_{i}), \boldsymbol{\Sigma}_{\boldsymbol{\theta}} (\vx_{S^b_{i}}, S^b_{i}))$.
% \kj{The $\boldsymbol{\mu}_{\boldsymbol{\theta}}$ and $\boldsymbol{\Sigma}_{\boldsymbol{\theta}}$ were previously used for single time-step reverse, and they could also be used for a larger time step reverse? Or they should be a composition/recursion of themselves multiple times?}

\section{More Experimental details  and  Results}
\subsection{Implementation details}
We select  three different noise levels $\sigma \in \left\{ 0.25, 0.5, 1.0 \right\}$ for certification. For the parameters of \name{}, 
The sampling numbers when computing the certified radius are $n = 100000$ for CIFAR-10 and $n = 10000$ for ImageNet. 
% We choose the same subset from \citet{Cohen2019ICML} for certification. 
We evaluate the certified robustness on 500 samples subset of CIFAR-10 testset and 500 samples subset of ImageNet validation set.
we set  $K = 40$ and $b$ = 10 except the results in ablation study.  The details about the baselines are in the appendix. 

\subsection{Baselines.}
We select randomized smoothing based methods including 
PixelDP~\citep{lecuyer2019certified}, RS~\citep{Cohen2019ICML}, SmoothAdv ~\citep{salman2019provably}, Consistency ~\citep{jeong2020consistency}, MACER ~\citep{zhai2020macer}, Boosting ~\citep{horvath2021boosting} , SmoothMix ~\citep{jeong2021smoothmix}, Denoised  ~\citep{salman2020denoised}, Lee~\citep{lee2021provable}, Carlini~\citep{carlini2022certified} as our baselines. 
Among them,  PixelDP, RS,  SmoothAdv, Consistency, MACER, and SmoothMix require  training
a smooth classifier for a better certification performance while the others do not.
% need to train the classifier. 
\citeauthor{salman2020denoised} and \citeauthor{lee2021provable} use the off-the-shelf classifier  but without using the diffusion model. 
The most similar one compared with us is \citeauthor{carlini2022certified}, which also uses both the off-the-shelf diffusion model and classifier.   
The above two settings mainly refer to \cite{carlini2022certified}, which makes us easier to compared with their results.

\subsection{Main Results for Certified Accuracy}\label{main}
We compare with \citet{carlini2022certified} in a more fine-grained version. We provide  results of certified accuracy at different $\eps$ in Table~\ref{tbl:cifar10ab} for CIFAR-10 and Table~\ref{tbl:imagenetab} for ImageNet. 
We include the accuracy difference between ours and ~\citet{carlini2022certified} in the bracket in Tables. 
We can observe from the tables that the certified accuracy of our method outperforms \citet{carlini2022certified} except $\eps = 0$ at $\sigma = 0.25, 0.5$ for CIFAR-10.

\begin{table}[t]
 \resizebox{\linewidth}{!}{%
\begin{tabular}{llrrrrr}
\toprule
          &               & \multicolumn{5}{c}{Certified Accuracy at $\boldsymbol{\eps}(\%)$}             \\ 
        %   \cline{3-7} 
Methods   & Noise         & 0.0        & 0.25       & 0.5        & 0.75       & 1.0        \\ \midrule
             & $\sigma = 0.25$ & \textbf{88.0}       & 73.8       & 56.2       & 41.6       & 0.0        \\
Carlini~\citep{carlini2022certified} & $\sigma = 0.5$  & 74.2       & 62.0       & 50.4       & 40.2       & 31.0       \\
          & $\sigma = 1.0$  & 49.4       & 41.4       & 34.2       & 27.8       & 21.8       \\ \midrule
                 & $\sigma = 0.25$ & 87.6(-0.4) & \textbf{76.6(+2.8)} & \textbf{64.6(+8.4)} & \textbf{50.4(+8.8)} & 0.0(+0.0)  \\
\textbf{Ours}      & $\sigma = 0.5$  & 73.6(-0.6) & 65.4(+3.4) & 55.6(+5.2) & 46.0(+5.8) & \textbf{37.4(+6.4)} \\
          & $\sigma = 1.0$  & 55.0(+5.6) & 47.8(+6.4) & 40.8(+6.6) & 33.0(+5.2) & 28.2(+6.4) \\ \bottomrule

\end{tabular}}
\caption{Certified accuracy compared with \cite{carlini2022certified} for CIFAR-10 at all $\sigma$. The numbers in the bracket are the difference of certified accuracy between two methods. Our diffusion model and classifier are the same as \cite{carlini2022certified}.}
\label{tbl:cifar10ab}
\end{table}

\begin{table}[t]
 \resizebox{\linewidth}{!}{%
\begin{tabular}{llrrrrrr}
\toprule
          &               & \multicolumn{6}{c}{Certified Accuracy at $\boldsymbol{\eps}(\%)$}             \\ 
        %   \cline{3-7} 
Methods   & Noise         & 0.0        & 0.5       & 1.0        & 1.5       & 2.0    &3.0        \\ \midrule
             & $\sigma = 0.25$ & 82.0       & 74.0       & 0.0       & 0.0       & 0.0   &0.0     \\
Carlini~\citep{carlini2022certified} & $\sigma = 0.5$  & 77.2       & 71.8       & 59.8      & 47.0       & 0.0 & 0.0      \\
          & $\sigma = 1.0$  & 64.6       & 57.8       & 49.2       & 40.6       & 31.0  & 19.0    \\ \midrule
                 & $\sigma = 0.25$ & \textbf{84.0(+2.0)} & \textbf{77.8(+3.8)} & 0.0(+0.0) & 0.0(+0.0) & 0.0(+0.0) &0.0(+0.0) \\
\textbf{Ours}      & $\sigma = 0.5$  & 80.2(+3.0) & 75.6(+3.8) & \textbf{67.0(+7.2)} & \textbf{54.6(+7.6)} & 0.0(+0.0) & 0.0(+0.0)\\
          & $\sigma = 1.0$  & 67.8(+3.2) & 61.4(+3.6) & 55.6(+6.4) & 50.0(+9.4) & \textbf{42.2(+11.2)} & \textbf{25.8(+6.8)} \\ \bottomrule

\end{tabular}}
\caption{Certified accuracy compared with \cite{carlini2022certified} for ImageNet at all $\sigma$. The numbers in the bracket are the difference of certified accuracy between two methods. Our diffusion model and classifier are the same as \cite{carlini2022certified}.}
\label{tbl:imagenetab}
\end{table}

\subsection{Experiments for Voting Samples}\label{exp:vote}
Here we provide more experiments with $\sigma \in \{ 0.5, 1.0\}$ and $b=10$ for different voting samples $K$ in Figure~\ref{fig:mv_0.5} and Figure~\ref{fig:mv_1.0}.  The results for CIFAR-10 is in Figure~\ref{fig:mv-cifar}. We can draw the same conclusion mentioned in the main context .
% certified accuracy increases with more majority vote numbers and finally converges.

\begin{figure*}[t]
\small
\centering
    \begin{minipage}{0.45\linewidth}
        \centering
        \includegraphics[width=\textwidth]{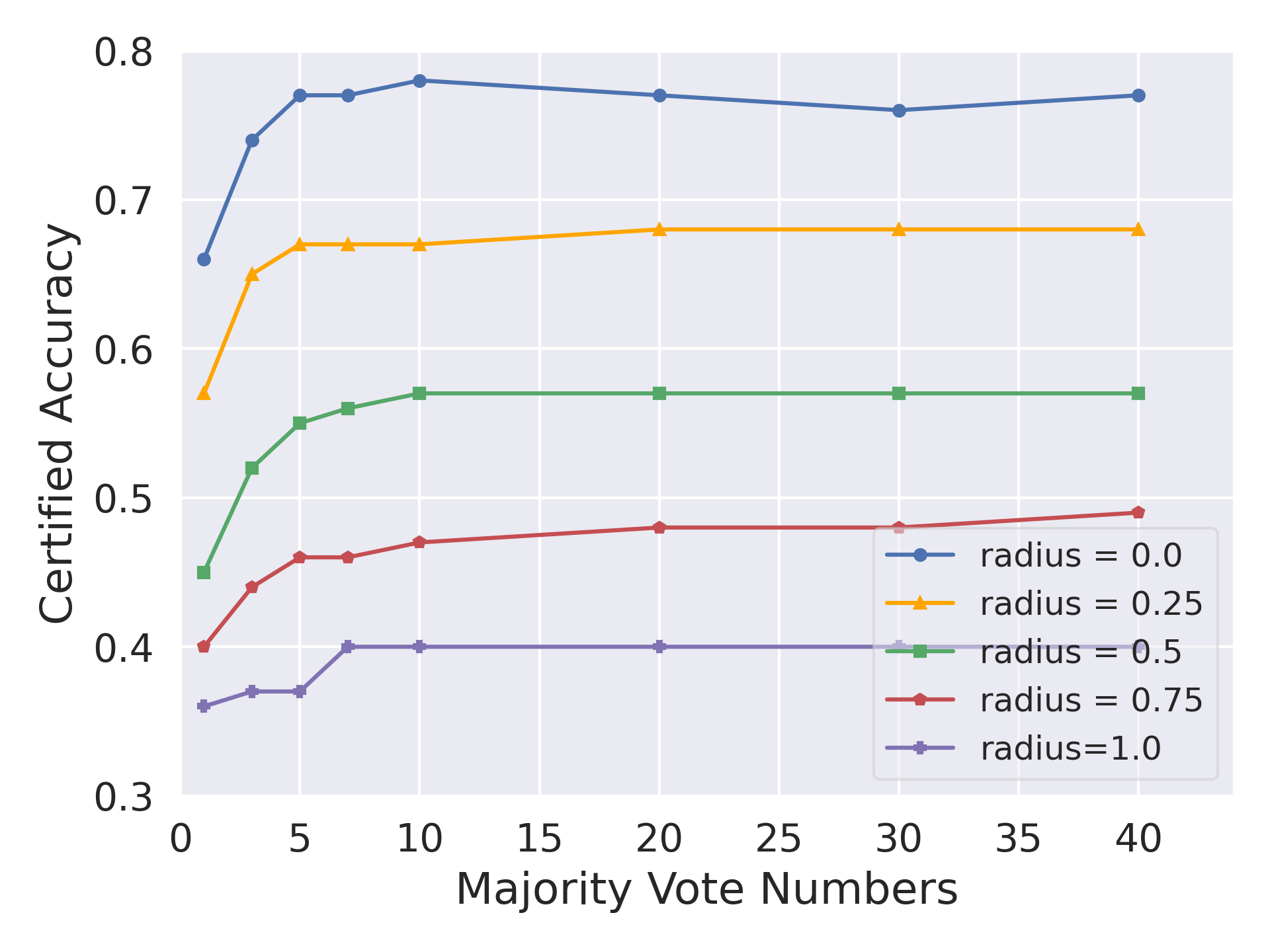}\\
        CIFAR=10
    \end{minipage}\hfill
    \begin{minipage}{0.45\linewidth}
        \centering
        \includegraphics[width=\textwidth]{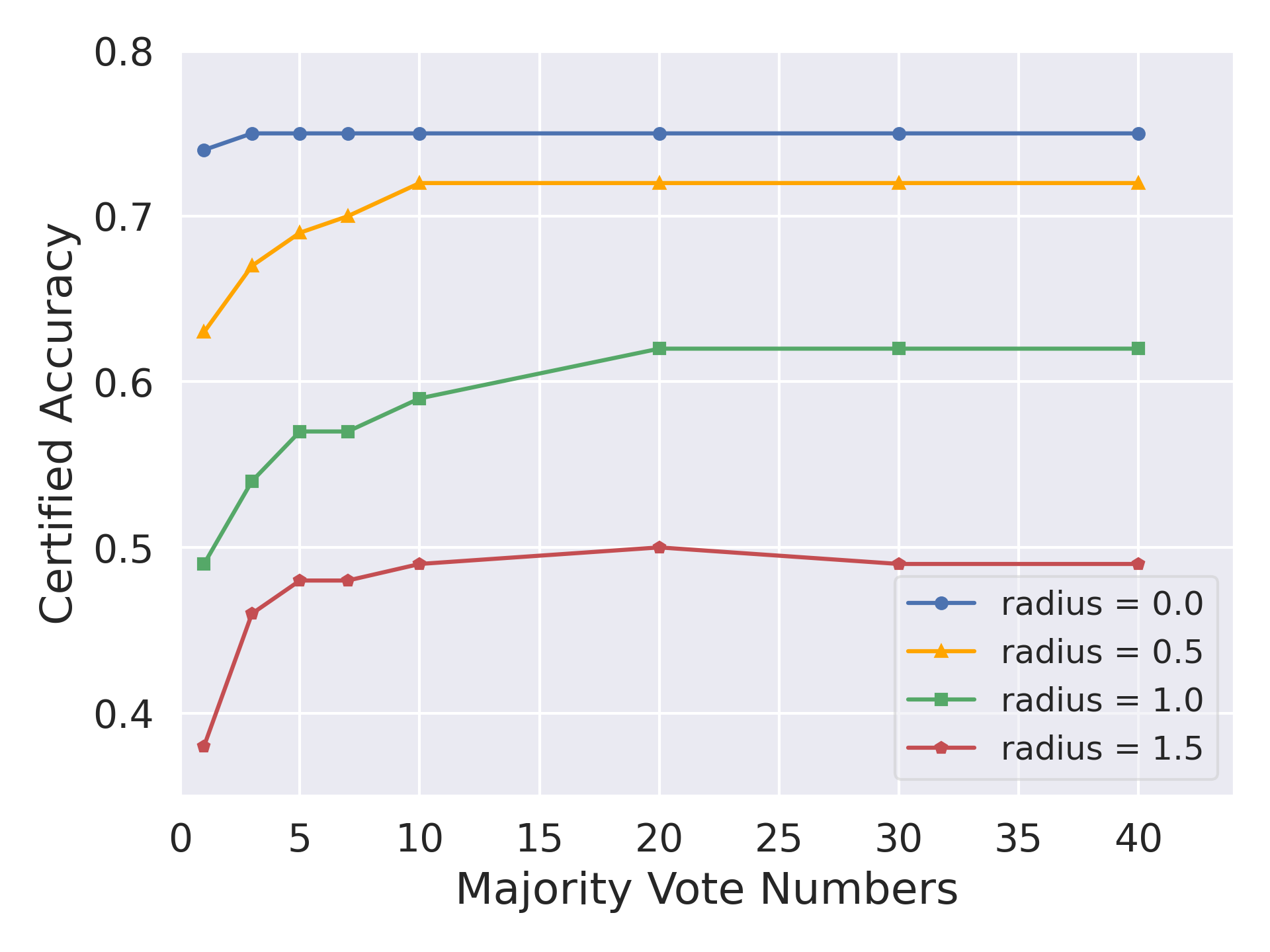}\\
        ImageNet
    \end{minipage}\hfill
    % \begin{minipage}{0.302\linewidth}
    %     \centering
    %     \includegraphics[width=\textwidth]{cifar10_500_comparation.png}\\
    %     (c)
    % \end{minipage}
\vspace{-2mm}
\caption{Certified accuracy among different vote numbers with different radius. Each line in the figure represents the certified accuracy among different vote numbers K with Gaussian noise $\sigma=0.50$.}
\label{fig:mv_0.5}
\end{figure*}

\begin{figure*}[t] 
\small
\centering
    \begin{minipage}{0.45\linewidth}
        \centering
        \includegraphics[width=\textwidth]{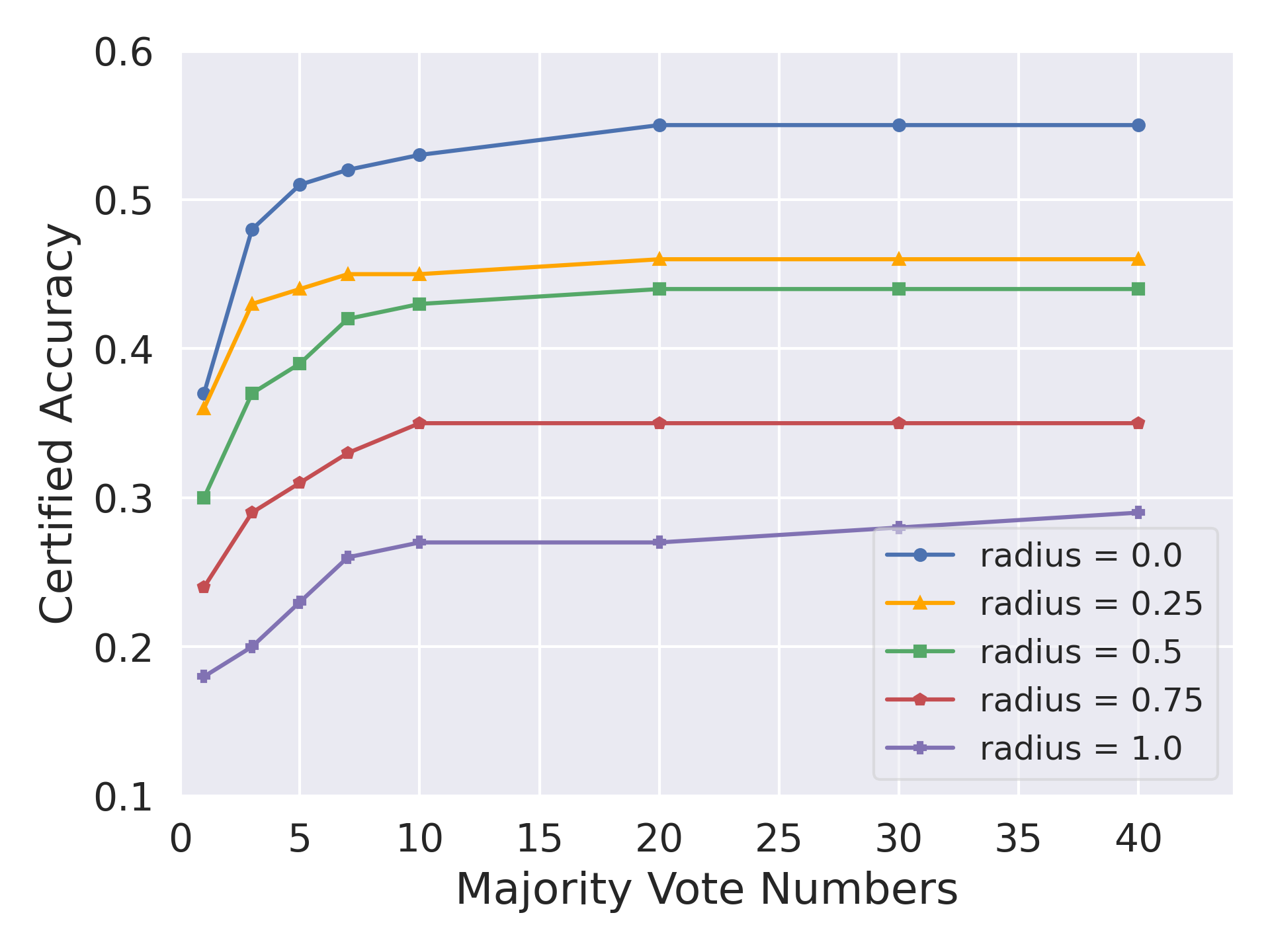}\\
        CIFAR=10
    \end{minipage}\hfill
    \begin{minipage}{0.45\linewidth}
        \centering
        \includegraphics[width=\textwidth]{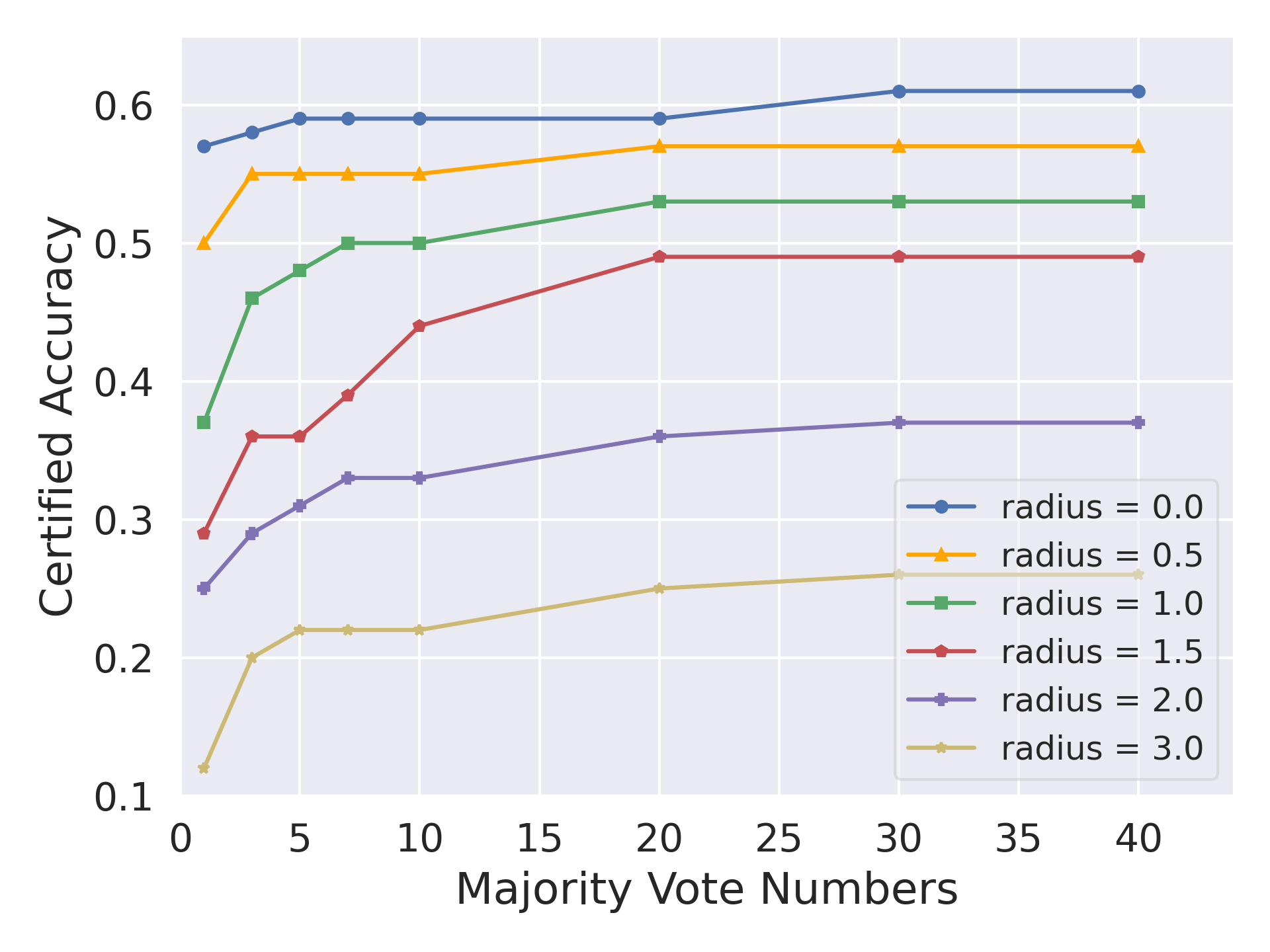}\\
        ImageNet
    \end{minipage}\hfill
    % \begin{minipage}{0.302\linewidth}
    %     \centering
    %     \includegraphics[width=\textwidth]{cifar10_500_comparation.png}\\
    %     (c)
    % \end{minipage}
\vspace{-2mm}
\caption{Certified accuracy among different vote numbers with different radius. Each line in the figure represents the certified accuracy among different vote numbers K with Gaussian noise $\sigma=1.00$.}
\label{fig:mv_1.0}
\end{figure*}

\subsection{Experiments for Fast Sampling Steps}\label{exp:steps}
We also implement additional experiments with $b \in \{1, 2, 10\}$ at $\sigma = 0.5, 1.0$. The results are shown in Figure~\ref{fig:steps_0.5} and Figure~\ref{fig:steps_1.0}. The results for CIFAR-10 are in Figure~\ref{fig:mv-cifar}. We draw the same conclusion as mentioned in the main context.

\begin{figure*}[t] 
\small
\centering
    \begin{minipage}{0.45\linewidth}
        \centering
        \includegraphics[width=\textwidth]{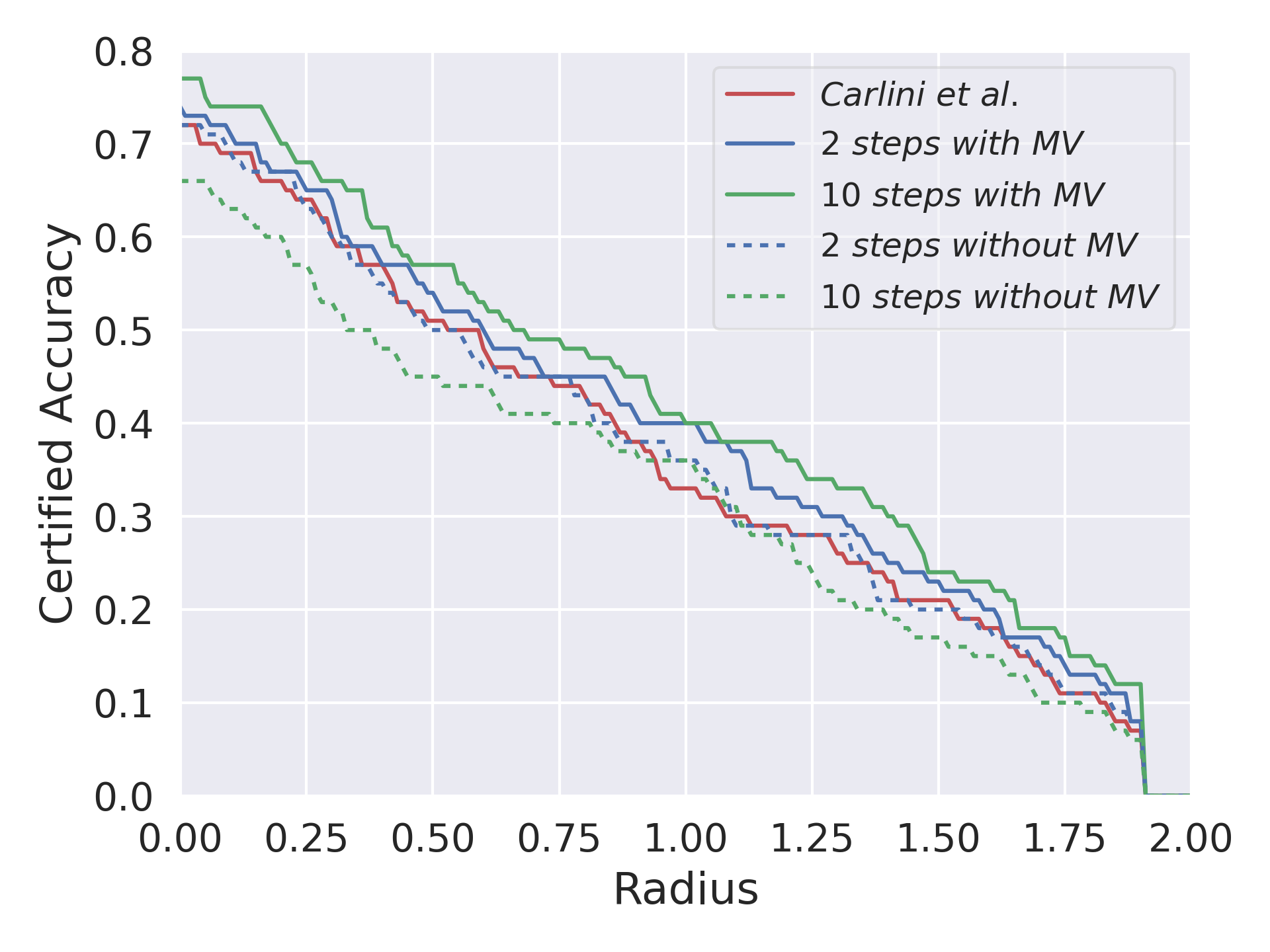}\\
        CIFAR=10
    \end{minipage}\hfill
    \begin{minipage}{0.45\linewidth}
        \centering
        \includegraphics[width=\textwidth]{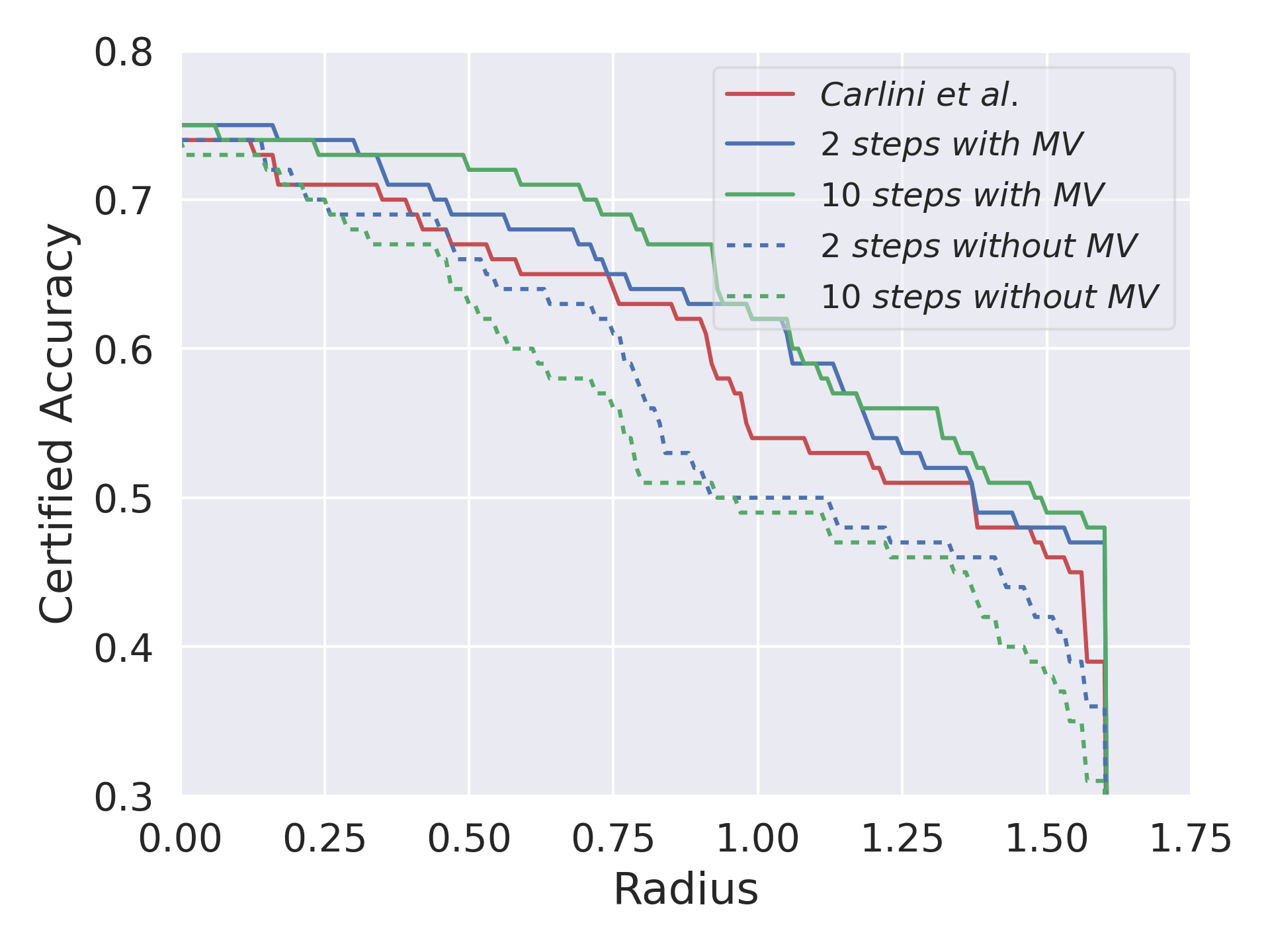}\\
        ImageNet
    \end{minipage}\hfill
    % \begin{minipage}{0.302\linewidth}
    %     \centering
    %     \includegraphics[width=\textwidth]{cifar10_500_comparation.png}\\
    %     (c)
    % \end{minipage}
\vspace{-2mm}
\caption{Certified accuracy with different fast sampling steps $b$. Each line in the figure shows the certified accuracy among different $L_2$ adversarial perturbation bound with Gaussian noise $\sigma=0.50$.}
\label{fig:steps_0.5}
\end{figure*}

\begin{figure*}[t] 
\small
\centering
    \begin{minipage}{0.45\linewidth}
        \centering
        \includegraphics[width=\textwidth]{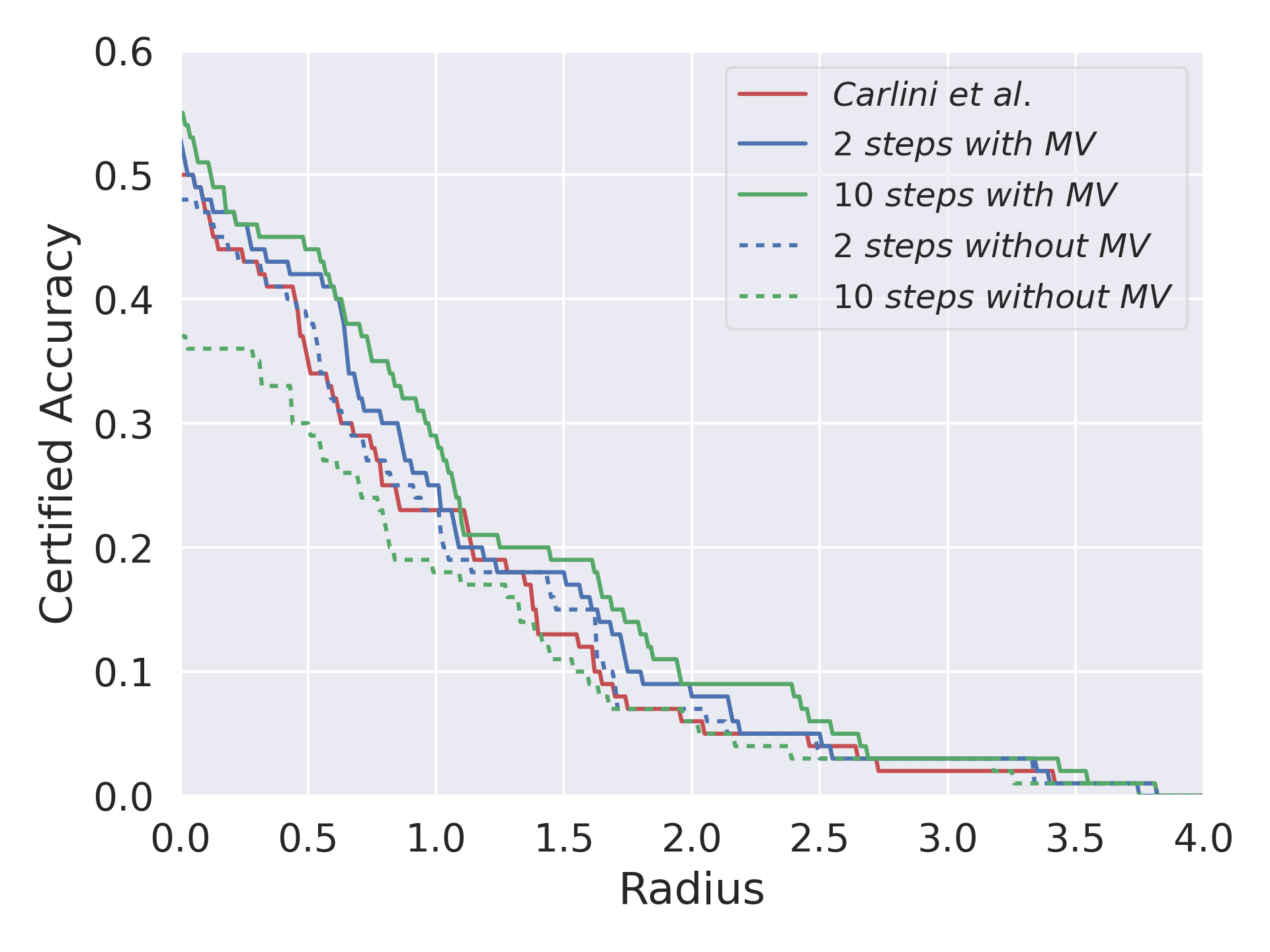}\\
        CIFAR=10
    \end{minipage}\hfill
    \begin{minipage}{0.45\linewidth}
        \centering
        \includegraphics[width=\textwidth]{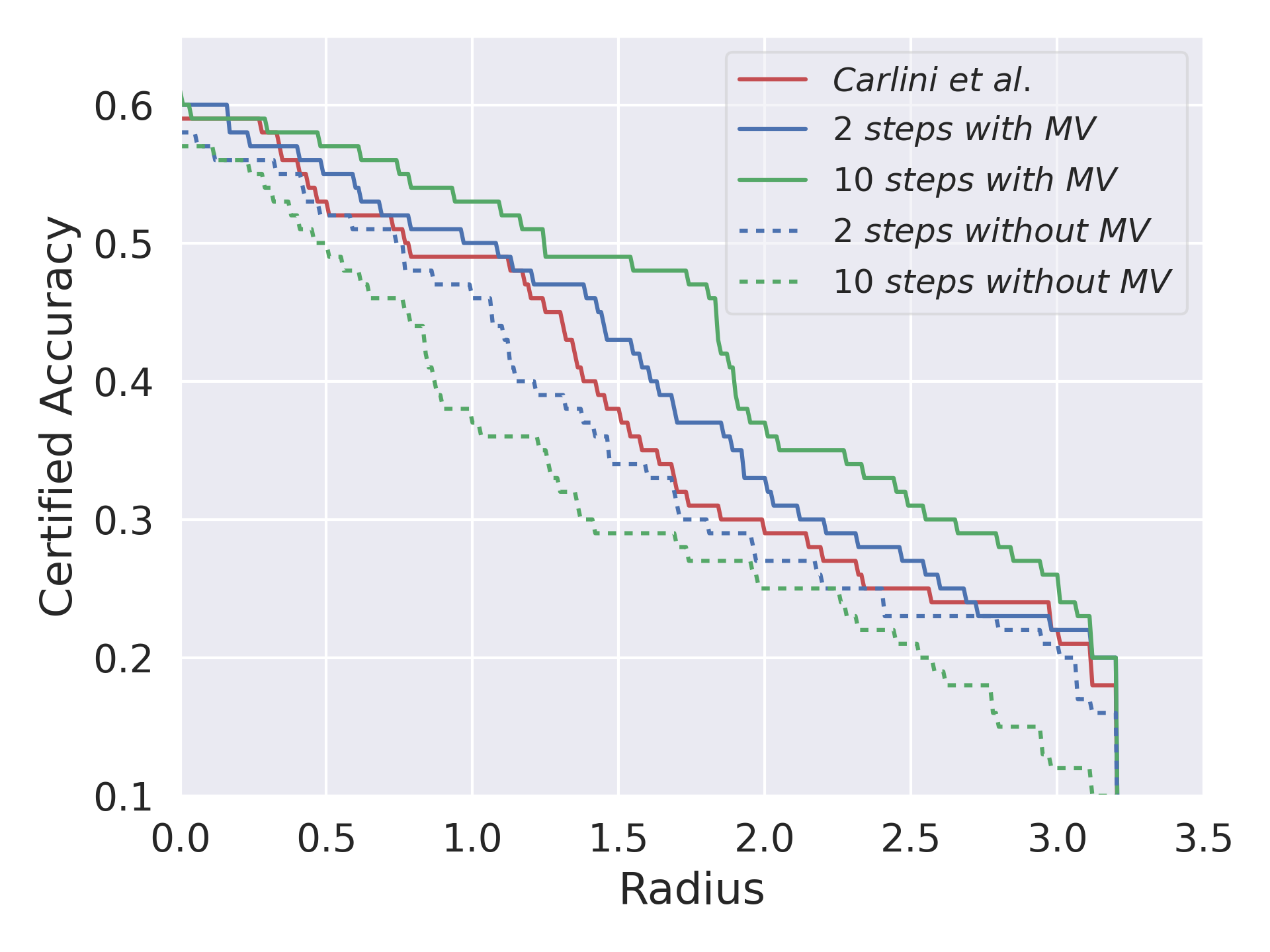}\\
        ImageNet
    \end{minipage}\hfill
    % \begin{minipage}{0.302\linewidth}
    %     \centering
    %     \includegraphics[width=\textwidth]{cifar10_500_comparation.png}\\
    %     (c)
    % \end{minipage}
\vspace{-2mm}
\caption{Certified accuracy with different fast sampling steps $b$. Each line in the figure shows the certified accuracy among different $L_2$ adversarial perturbation bound with Gaussian noise $\sigma=1.00$.}
\label{fig:steps_1.0}
\end{figure*}

\subsection{Experiments for Different Architectures}\label{exp:models}
We try different model architectures of ImageNet including Wide ResNet-50-2 and ResNet 152 with $b=2$ and $K=10$. The results are shown in Figure~\ref{fig:wrn}. 
% For simplicity, we only choose sampling steps $b = 2$ and majority vote number $K = 10$. For both model Wide ResNet-50-2 and ResNet152, 
we find that our method outperforms ~\citep{carlini2022certified} for all $\sigma$ among different classifiers. 

\begin{figure*}[t] 
\small
\centering
    \begin{minipage}{0.45\linewidth}
        \centering
        \includegraphics[width=\textwidth]{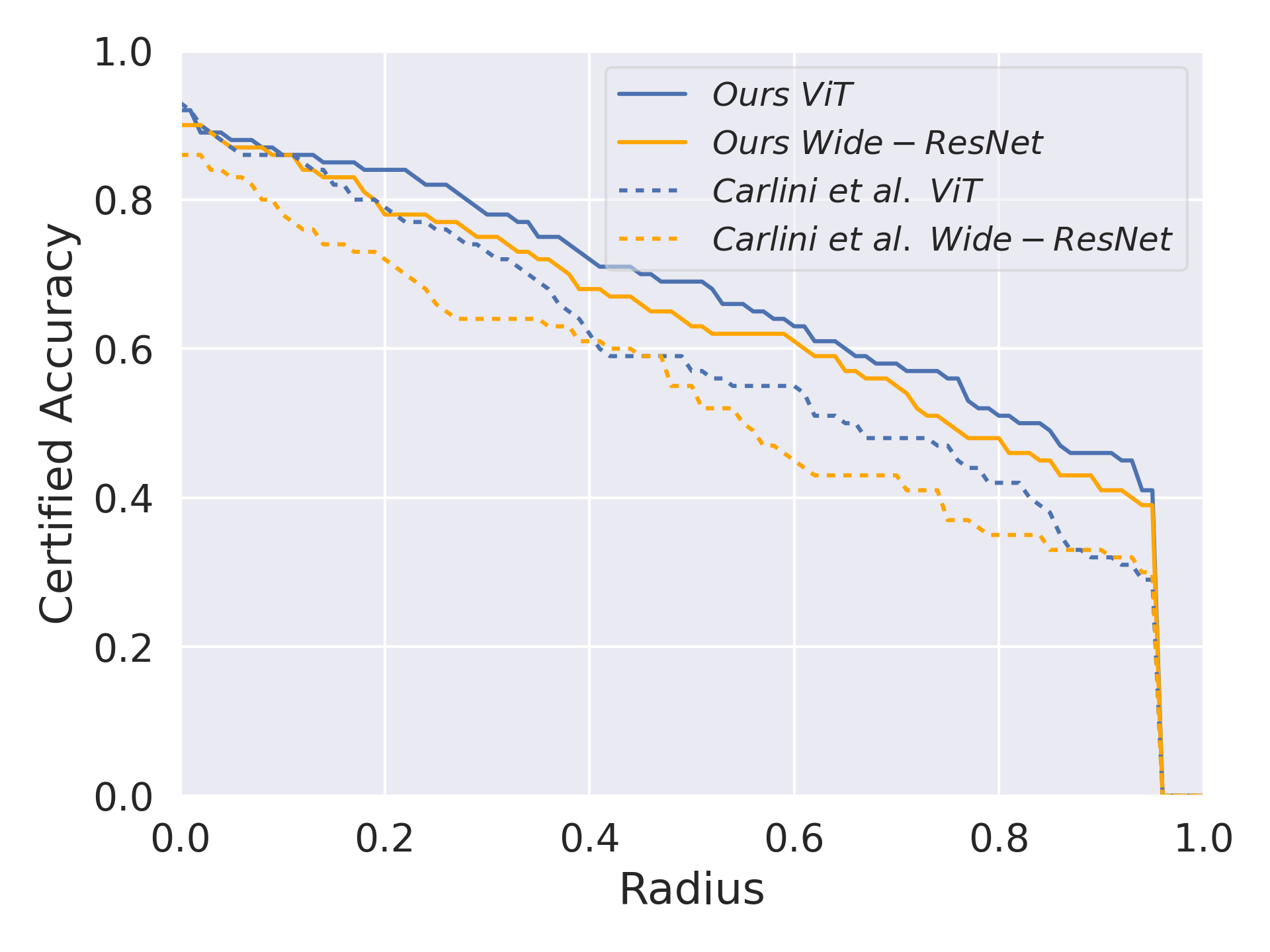}\\
        CIFAR=10
    \end{minipage}\hfill
    \begin{minipage}{0.45\linewidth}
        \centering
        \includegraphics[width=\textwidth]{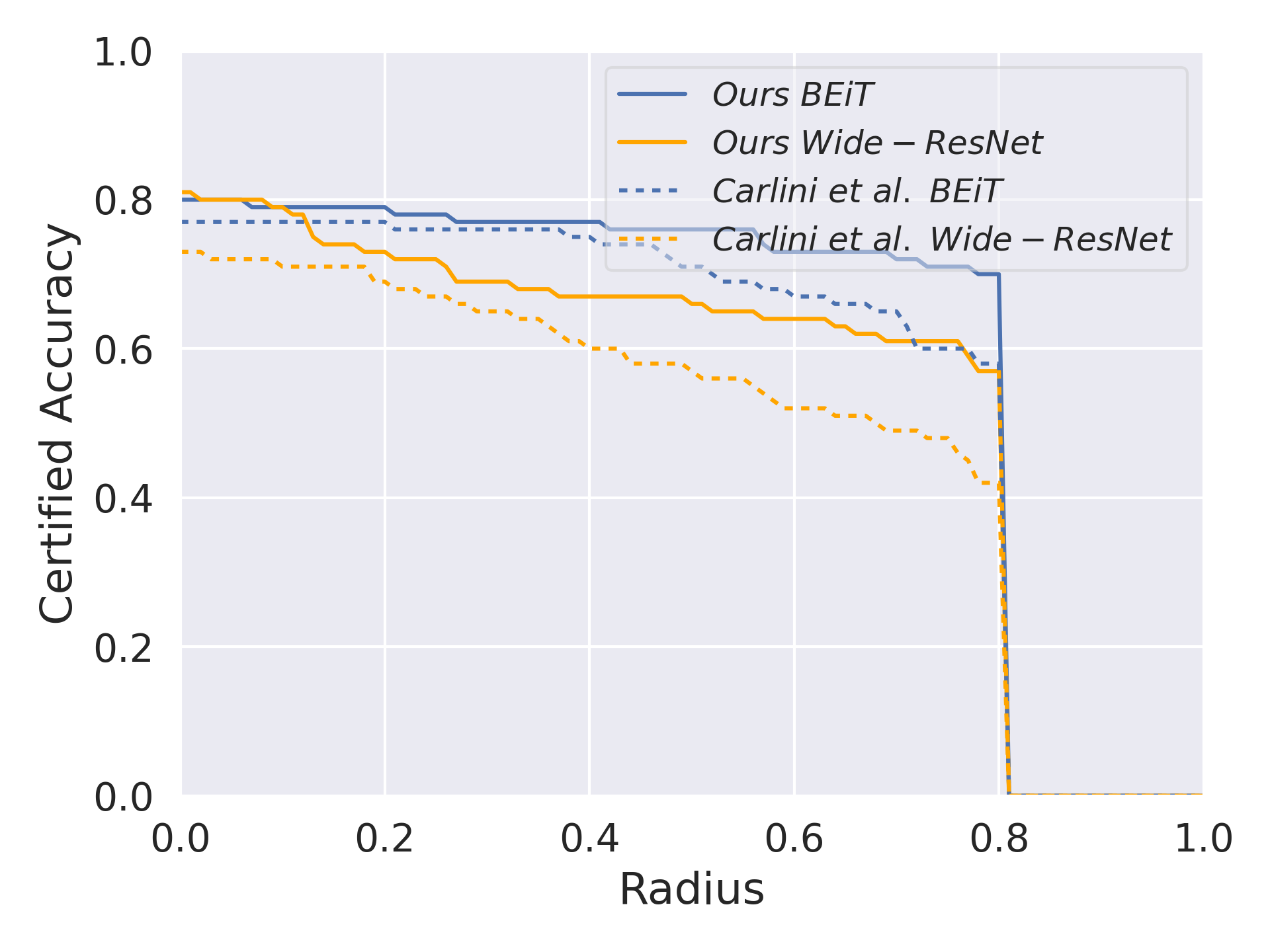}\\
        ImageNet
    \end{minipage}\hfill
    % \begin{minipage}{0.302\linewidth}
    %     \centering
    %     \includegraphics[width=\textwidth]{cifar10_500_comparation.png}\\
    %     (c)
    % \end{minipage}
\vspace{-2mm}
\caption{Certified accuracy with different architectures. Each line in the figure shows the certified accuracy among different $L_2$ adversarial perturbation bound with Gaussian noise $\sigma=0.25$.}
\label{fig:modelarch}
\end{figure*}

\begin{figure*}[t] 
\small
\centering
    \begin{minipage}{0.45\linewidth}
        \centering
        \includegraphics[width=\textwidth]{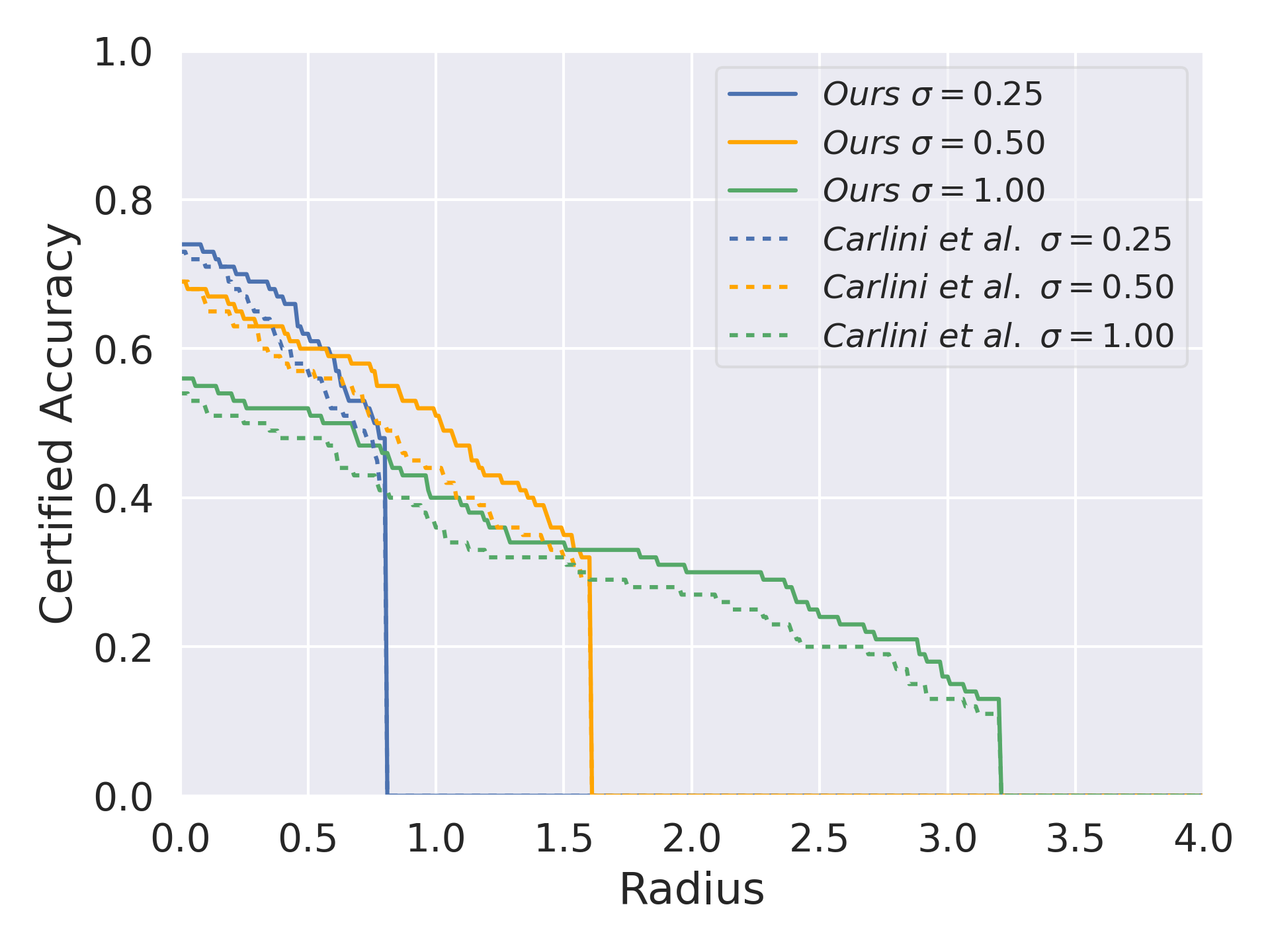}\\
        Wide ResNet-50-2
    \end{minipage}\hfill
    \begin{minipage}{0.45\linewidth}
        \centering
        \includegraphics[width=\textwidth]{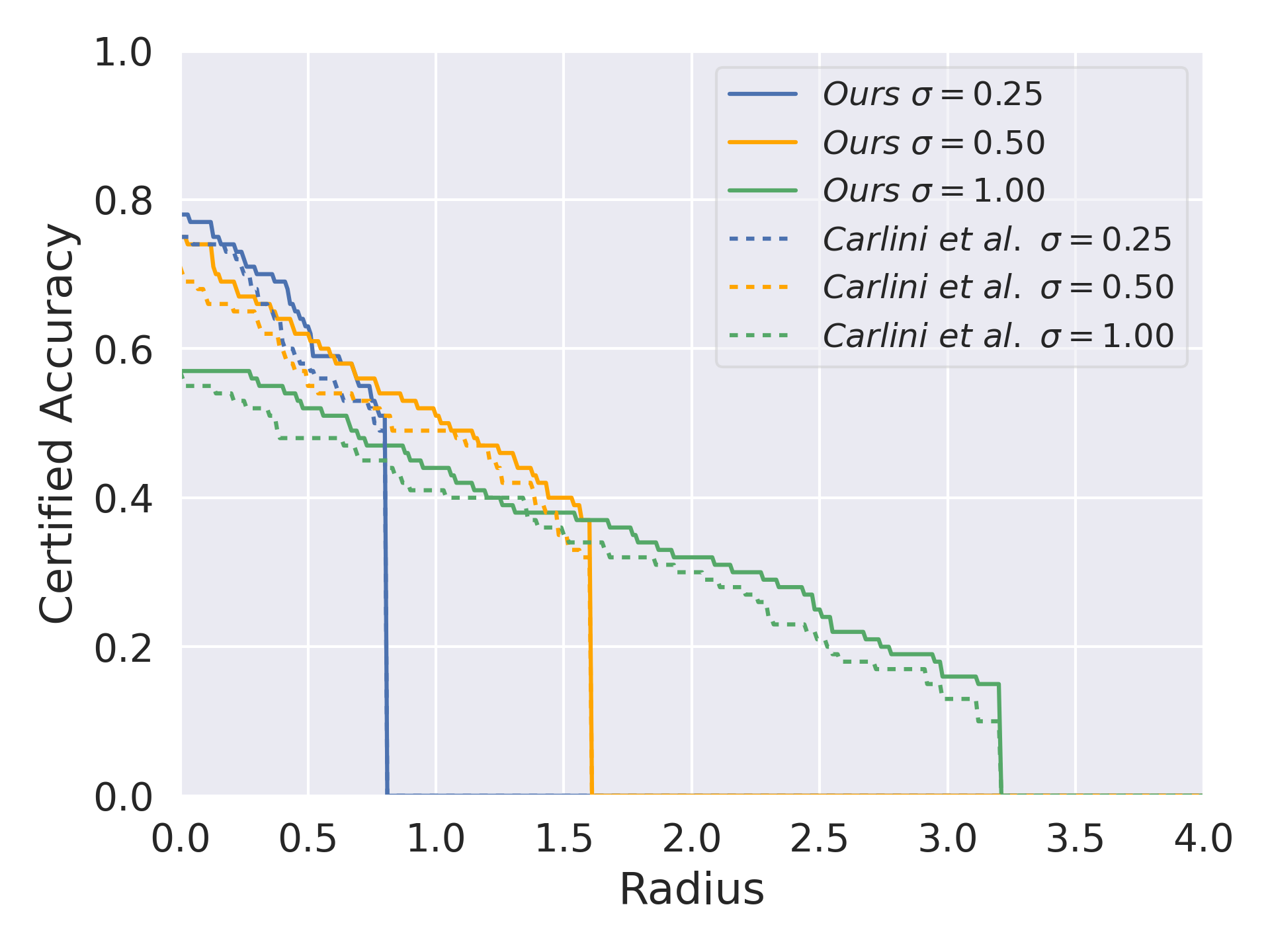}\\
        ResNet152
    \end{minipage}\hfill
    % \begin{minipage}{0.302\linewidth}
    %     \centering
    %     \includegraphics[width=\textwidth]{cifar10_500_comparation.png}\\
    %     (c)
    % \end{minipage}
\vspace{-2mm}
\caption{Certified accuracy of ImageNet for different architectures. The lines represent the certified accuracy with different $L_2$ perturbation bound with different Gaussian noise $\sigma \in \{0.25, 0.50, 1.00\}$.}
\label{fig:wrn}
\end{figure*}

\begin{figure*}[t] 
\small
\centering
    % \begin{minipage}{0.24\linewidth}
    %     \centering
    %     \includegraphics[width=\textwidth]{figures/cifar10_mv.png}\\
    %     CIFAR=10
    % \end{minipage}
    \begin{minipage}{0.45\linewidth}
        \centering
        \includegraphics[width=\textwidth]{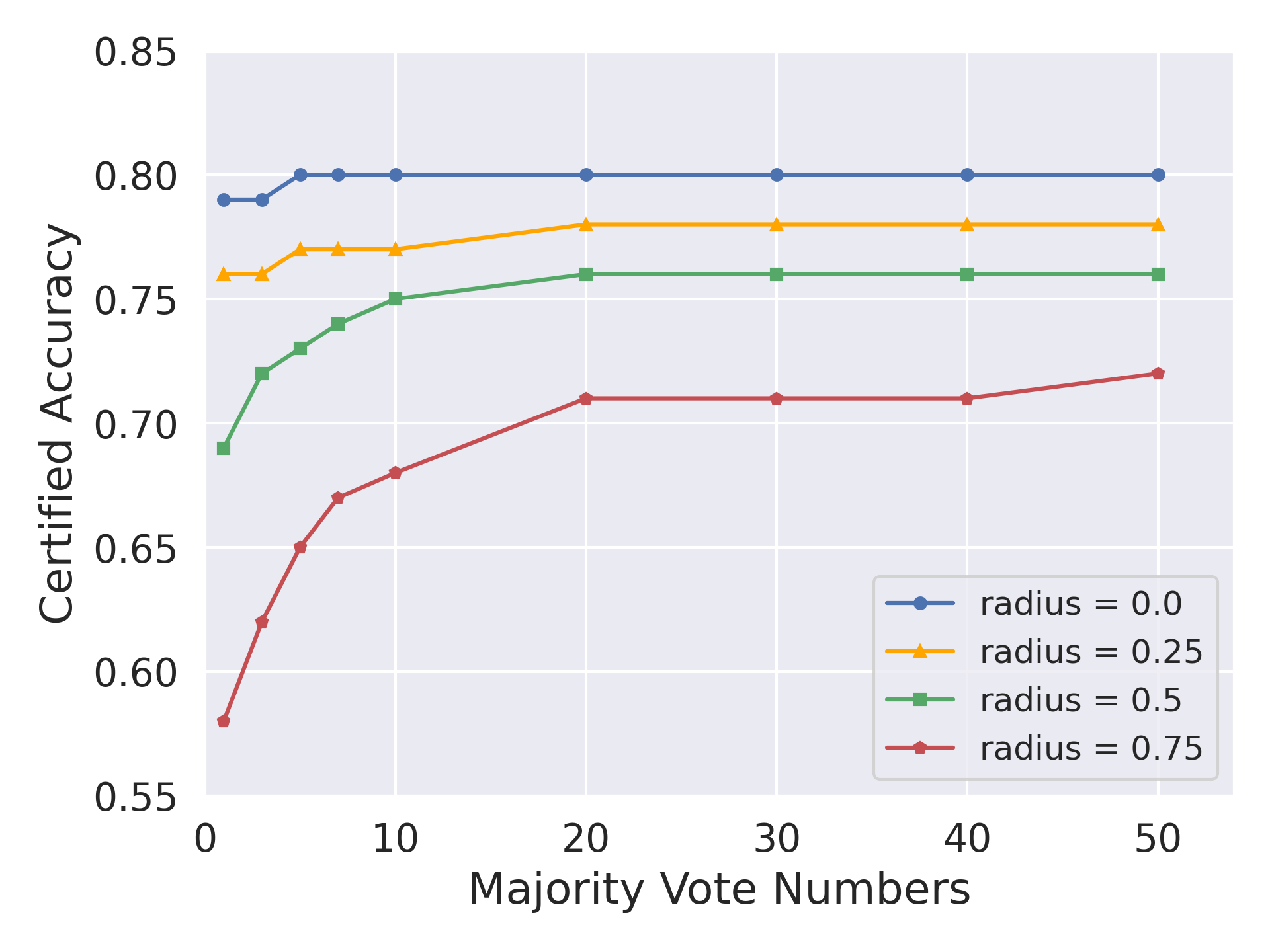}\\
        ImageNet
    \end{minipage}
    %   \begin{minipage}{0.24\linewidth}
    %     \centering
    %     \includegraphics[width=\textwidth]{figures/cifar10_steps.png}\\
    %     CIFAR=10
    % \end{minipage}
    \begin{minipage}{0.45\linewidth}
        \centering
        \includegraphics[width=\textwidth]{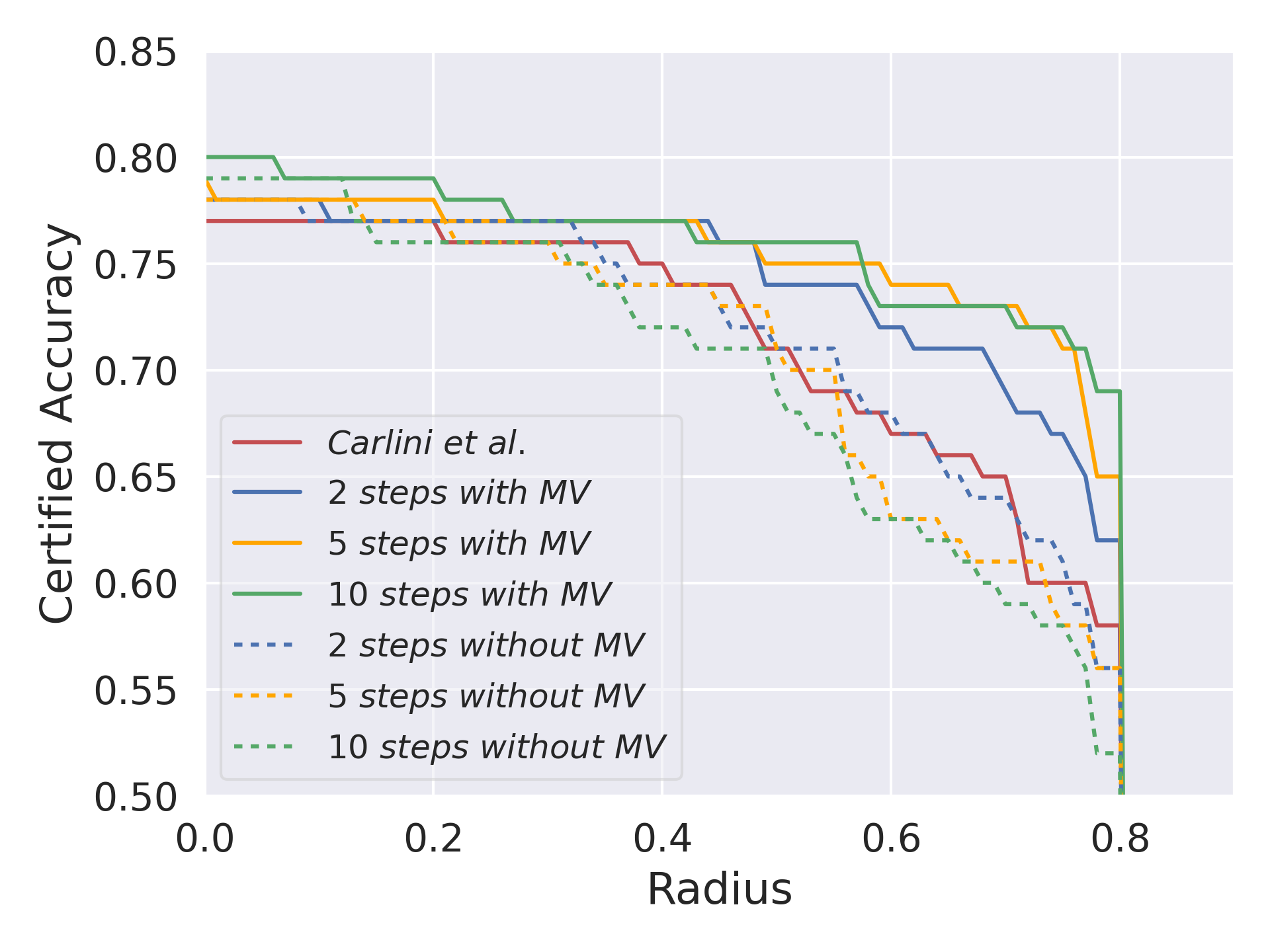}\\
        ImageNet
    \end{minipage}
\vspace{-2mm}
\caption{Ablation study. The left image shows the certified accuracy among different vote numbers with different radius $\epsilon \in \{0.0, 0.25, 0.5, 0.75\}$. Each line in the figure represents the certified accuracy of our method among different vote numbers $K$ with Gaussian noise $\sigma=0.25$. The right image shows the certified accuracy with different fast sampling steps $b$. Each line in the figure shows the certified accuracy among different $L_2$ adversarial perturbation bound.}
\label{fig:mv-cifar}
\end{figure*}

\end{document}